\pdfoutput=1
\documentclass[letterpaper]{article} 
\usepackage{aaai24}  
\usepackage{times}  
\usepackage{helvet}  
\usepackage{courier}  
\usepackage[hyphens]{url}  
\usepackage{graphicx} 
\usepackage{natbib}  
\usepackage{caption} 
\usepackage{algorithm}
\usepackage{algorithmic}

\usepackage{newfloat}
\usepackage{listings}

\usepackage{bm} 
\usepackage{amsthm,  amsmath,  amssymb} 
\usepackage{mathrsfs}

\usepackage{textcomp}
\usepackage{pifont} 

\usepackage{diagbox} 
\usepackage{multirow} 
\usepackage{booktabs} 

\usepackage{graphicx} 
\usepackage{float}  
\usepackage{subfig} 
\usepackage{caption}

\newtheorem{theorem}{Theorem}

\newtheorem{lemma}{Lemma}

\DeclareCaptionStyle{ruled}{labelfont=normalfont,labelsep=colon,strut=off} 
\floatstyle{ruled}
\newfloat{listing}{tb}{lst}{}
\floatname{listing}{Listing}
\title{Multi-class Support Vector Machine with Maximizing Minimum Margin}
\author{
	Zhezheng Hao\textsuperscript{\rm 1,2}\footnotemark[1],
    Feiping Nie\textsuperscript{\rm 1,2}\thanks{Equal Contribution.}\thanks{Corresponding author.}, 
    Rong Wang\textsuperscript{\rm 2}
}
\affiliations{
    \textsuperscript{\rm 1}School of Cybersecurity, Northwestern Polytechnical University, Xi’an, China\\
    \textsuperscript{\rm 2} School of Artificial Intelligence, Optics and Electronics (iOPEN), Northwestern Polytechnical University, Xi’an, China

}

\usepackage{bibentry}

\begin{document}

\maketitle

\begin{abstract}
Support Vector Machine (SVM) stands out as a prominent machine learning technique widely applied in practical pattern recognition tasks. 
It achieves binary classification by maximizing the "margin", which represents the minimum distance between instances and the decision boundary.
Although many efforts have been dedicated to expanding SVM for multi-class case through strategies such as one versus one and one versus the rest, satisfactory solutions remain to be developed.
In this paper, we propose a novel method for multi-class SVM that incorporates pairwise class loss considerations and maximizes the minimum margin.
Adhering to this concept, we embrace a new formulation that imparts heightened flexibility to multi-class SVM.
Furthermore, the correlations between the proposed method and multiple forms of multi-class SVM are analyzed.
The proposed regularizer, akin to the concept of "margin", can serve as a seamless enhancement over the softmax in deep learning, providing guidance for network parameter learning.
Empirical evaluations demonstrate the effectiveness and superiority of our proposed 
method over existing multi-classification methods.
Code is available at https://github.com/zz-haooo/M3SVM.
\end{abstract}

\section{Introduction}
Support vector machine, a fundamental machine learning technique, initially emerged as a binary linear classifier \cite{Boser1992}. 
Rooted in the theory of VC-dimension and the principle of structural risk minimization, SVM operates as a binary classifier delineated by a hyperplane.
Its mathematical rigor and notable performance in practical applications have garnered significant attention.
SVM-based classification methods have found extensive application in diverse machine learning tasks, including image classification \cite{Wei2016}, text classification \cite{Nie2014}, etc.
Diverse SVM variants have emerged over time, such as Twin SVM \cite{Khemchandani2007}, Optimal Margin Distribution Machine \cite{Zhang2019},Decision Tree SVM \cite{Nie2020}, etc.
Furthermore, SVM has been extended to encompass unsupervised and semi-supervised learning scenarios \cite{Xu2005, Amer2013}.

Despite its outstanding performance in binary classification tasks, research of SVM on multi-class classification progresses sluggishly.
Existing multi-class SVM can be summarized into two main categories, the first of which are One versus the Rest (OvR) \cite{Vapnik1998} and One versus One (OvO) techniques \cite{Hsu2002}.
For a $c$-class classification, OvR utilizes $c$ binary SVMs for hyperplanes between each class and the rest, while OvO utilizes $\frac{c(c-1)}{2}$ binary SVMs for hyperplanes between each pair of classes.
A critical shortcoming of OvR arises from the unbalancedness of subproblem.
Besides, treating the rest of the classes as a single class may lead to potential inseparability \cite{Pisner2020}.
The drawback of OvO lies in the large time and space overhead when there are numerous classes.
For $c$-class data, each sample is associated with $c-1$ hyperplanes, introducing unnecessary redundancy.
Therefore, OvO necessitates $\mathcal{O}(c^{2}d)$ for each test sample.
Since OvR and OvO involve multiple independent binary subproblems, they fail to achieve a complete partition of the feature space, resulting in certain regions with constant tied votes.
Additionally, decomposing the multi-class problem into multiple sub-problems leads to independent hyperparameters for each sub-problem,  posing challenges for fine-tuning.

The second major method is the direct separation of multi-class data under a unified optimization formulation \cite{Weston1998, Crammer2001, Guermeur2002}.
While these multi-class methods achieve sound performance on well-structured data, they may be less effective for datasets with ambiguous class boundaries.
Moreover, many of these multi-class SVM methods deviate from the fundamental principle of margin, thus impacting generalization capabilities.

To address the aforementioned issues, we propose a novel method called Multi-Class \textbf{S}upport \textbf{V}ector \textbf{M}achines with \textbf{M}aximizing \textbf{M}inimum \textbf{M}argin (\textbf{M$^3$SVM}) to overcome the limitations of existing methods.
The main contributions of this paper are summarized as follows:
\begin{itemize}
	\item In this paper, we propose a concise yet effective multi-class SVM method that computes classification loss between each class pair.
	More importantly, we derive a novel regularizer that enlarges the lower bound of the margin by introducing a parameter $p$. 

	\item The proposed method offers a lucid geometric interpretation and addresses issues prevalent in existing methods, such as the imbalanced subproblem in OvR, the redundancy of OvO, the reciprocal constraints within the unified formulations.
	It also functions as a plug-and-play improvement over the softmax in neural networks.
	We theoretically analyze the association of M$^3$SVM with the previous methods.
	Besides, the proposed regularizer can be interpreted in the context of minimizing structural risk.
	
	\item 
	Through exhaustive experimentation on  realistic datasets, our proposed method demonstrates a marked enhancement in classification performance.
	Moreover, experimental results in the realm of deep learning underscore its contribution to robustness and overfitting prevention.
\end{itemize}

The proofs of all involved lemmas and theorems together with supplementary experiments are relegated to Appendix.

\noindent \itshape Notations: \upshape The vectors and matrices are denoted by bold lowercase and bold uppercase letters, respectively. 
The transpose, the $i$-th column and the Frobenius norm of matrix $\mathbf{W}$ are denoted by $\mathbf{W}^{T}$, $\mathbf{w}_{i}$, $\|\mathbf{W}\|_{F}$, respectively.
Set $\{1,2,\cdots,n\}$ is abbreviated by $[n]$ for simplicity.

\section{Related Work}
\subsection{Binary SVM}
SVM was first proposed under the form of hard margin, whose basic formulation can be written as the following optimization problem \cite{Boser1992}:
\begin{equation}
		\min_{\mathbf{w},b} \frac{1}{2}\|\mathbf{w}\|_2^2, \quad
		\text { s.t. } y_{i}\left(\mathbf{w}^T \mathbf{x}_{i}+b\right) \geqslant 1, i \in [n],
\end{equation}
where $\{(\mathbf{x}_{i},y_{i})\}_{i = 1}^{n}$ is the train set.
SVM aims to seek a separating hyperplane by maximizing the margin between two classes such that $\mathbf{w}^T \mathbf{x}_{i}+b-1 \geqslant 0$ with $y_{i} = 1$ and $\mathbf{w}^T \mathbf{x}_{i}+b +1 \leqslant 0$ with $y_{i} = -1$.
The principle of SVM is demonstrated in Figure \ref{Binary SVM}.
It maximizes the "margin" in the figure through searching the appropriate support vectors.
The separating hyperplane $\mathbf{w}^{T}\mathbf{x} + b = 0$ obtained by solving the quadratic programming problem above can be employed to classify the upcoming data without label.
Cortes and Vapnik \cite{Cortes1995} proposed SVM with soft margin by introducing slack variables,
\begin{equation}\label{slack}
	\begin{aligned}
		&\min_{\mathbf{w},b} \sum_{i = 1}^{n}\xi_{i} + \lambda \|\mathbf{w}\|_2^2, \\
		&\text { s.t. } y_{i}\left(\mathbf{w}^T \mathbf{x}_{i}+b\right) \geqslant 1 - \xi_{i}, i \in [n].
	\end{aligned}
\end{equation}
Each slack variable corresponds to one sample, depicting the degree of unsatisfied constraints.
By substituting the slack variables, Eq. (\ref{slack}) can be formulated as the form of "loss $+$ regularization".
Thus, the general form of soft margin SVM with hinge loss is as follows:
\begin{equation}
	\min_{\mathbf{w},b} \sum_{i = 1}^{n} [1 - y_{i}(\mathbf{w}^T \mathbf{x}_{i}+b)]_{+} + \lambda\|\mathbf{w}\|_2^2.
\end{equation}
where $[\cdot]_{+} = max\{0,\cdot\}$.
$\lambda$ is a trade-off parameter to weigh the two objectives. 
Subsequent studies have refined the vanilla binary SVM in multiple ways \cite{Crisp1999, Mangasarian2001, Grandvalet2008, Ladicky2011, Zhou2012, Nie2014, Zhang2017}.
Nevertheless, the exploration of multi-class SVM remains incomplete.
A few representative methods are outlined below.

\subsection{Multi-class SVM with Unified Formulation}

When contemplating the decision function of multi-class linear models, an intuitive way of decision making is
\begin{equation}\label{decision}
	y = \emph{arg} \max_{k \in [c]} \mathbf{w}_{k}^T \mathbf{x}+b_{k}.
\end{equation}
In this way, each class corresponds to one projection vector $\mathbf{w}$.
In accordance with such decision function, a series of multi-class SVM models have been proposed.
Weston and Watkins \cite{Weston1998} integrated multi-class SVM into a unified framework rather than solving multiple subproblems separately, which is formulated as
\begin{equation}\label{Weston}
	\begin{aligned}
		&\min _{\mathbf{W} \in \mathbb{R}^{d \times c}, \mathbf{b} \in \mathbb{R}^c}    \sum_{i=1}^n \sum_{k \neq y_i} \xi_{ki} + \lambda \sum_{k=1}^c\left\|\mathbf{w}_{k}\right\|_2^2 \\
		&{ s.t. }\left\{\begin{array}{l}
			\mathbf{w}_{y_i}^T \mathbf{x}_i+b_{y_i} \geq \mathbf{w}_{k}^T \mathbf{x}_i+b_{k}+1-\xi_{ki}, \\
			\xi_{ki} \geqslant 0,k \in [c], k \neq y_{i},i \in [n].
		\end{array} \right.
	\end{aligned}
\end{equation}
where $\mathbf{W} \in \mathbb{R}^{d \times c}$ and $\mathbf{b} \in \mathbb{R}^c$ are applied for test data through Eq. (\ref{decision}).
Crammer and Singer \cite{Crammer2001} proposed a new-look loss function from the perspective of decision function, which is formulated as follows.
\begin{equation}\label{Crammer}
	\begin{aligned}
		&\min _{\mathbf{W} \in \mathbb{R}^{d \times c}, \mathbf{b} \in \mathbb{R}^c}   \sum_{i=1}^n \xi_{i} +  \lambda \sum_{k=1}^c\left\|\mathbf{w}_{k}\right\|_2^2 \\
		&{ s.t. }\mathbf{w}_{y_i}^T \mathbf{x}_i + \delta_{y_{i},k} - \mathbf{w}_{k}^T \mathbf{x}_i\geqslant 1 - \xi_{i},k \in [c],i \in [n].
	\end{aligned}
\end{equation}
where $\delta_{y_{i},k}$ equals to $1$ if $k = y_{i}$ and $0$ otherwise.

Denoting $[1 - (\mathbf{w}_{y_{i}}^{T} \mathbf{x}_{i} + b_{y_{i}} - \mathbf{w}_{k}^{T} \mathbf{x}_{i} - b_{k} )]_{+}$ as $\Delta_{ik}$, the optimization objective  in Eq. (\ref{Weston}) can be converted to
\begin{equation}\label{hinge_Weston}
	\min _{\mathbf{W} \in \mathbb{R}^{d \times c}, \mathbf{b} \in \mathbb{R}^c}  \sum_{i=1}^n \sum_{k \neq y_i} \Delta_{ik} + \lambda \sum_{k=1}^c\left\|\mathbf{w}_{k}\right\|_2^2.
\end{equation}
By derivation, the objective in Eq. (\ref{Crammer})  can be rewritten as
\begin{equation}
	\min _{\mathbf{W} \in \mathbb{R}^{d \times c}, \mathbf{b} \in \mathbb{R}^c}   \sum_{i=1}^n \max_{k \neq y_i} \Delta_{ik} + \lambda \sum_{k=1}^c\left\|\mathbf{w}_{k}\right\|_2^2.
\end{equation}
This makes it clear that per-sample loss in Eq. (\ref{Weston}) is the sum of loss of misclassification while that in Eq. (\ref{Crammer}) is the maximum loss.

\begin{figure*}[t]
	\centering
	\subfloat[Binary SVM.]{
		\includegraphics[height=2.5cm, width=3.8cm]{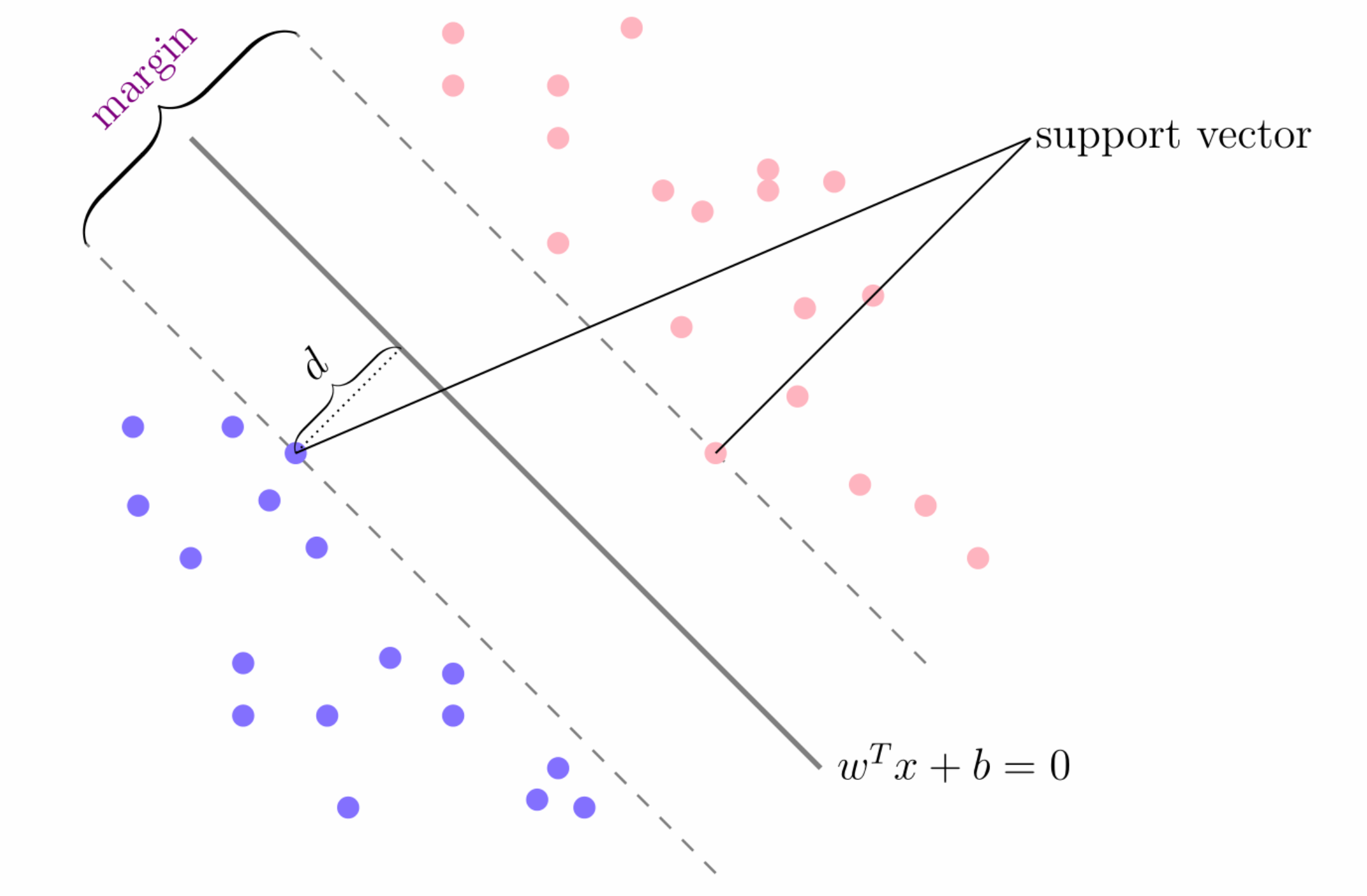} \label{Binary SVM}}  
	\subfloat[Multi-class SVM with three classes.]{
		\includegraphics[ height=2.5cm, width=5cm]{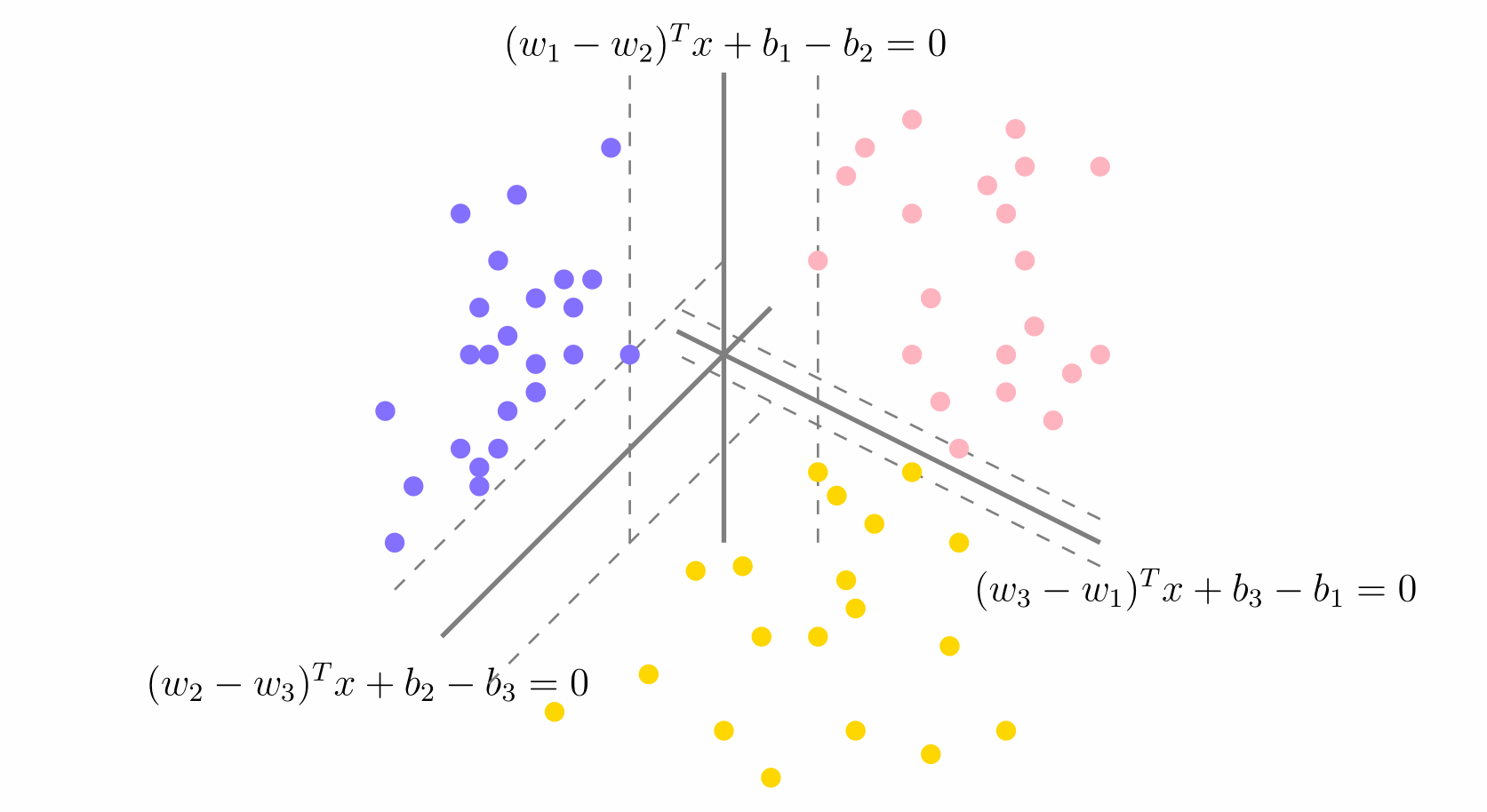} 
		\label{Piecewise-linear separator}}
	\subfloat[Iris with LDA.]{
		\includegraphics[height=2.5cm, width=4.2cm]{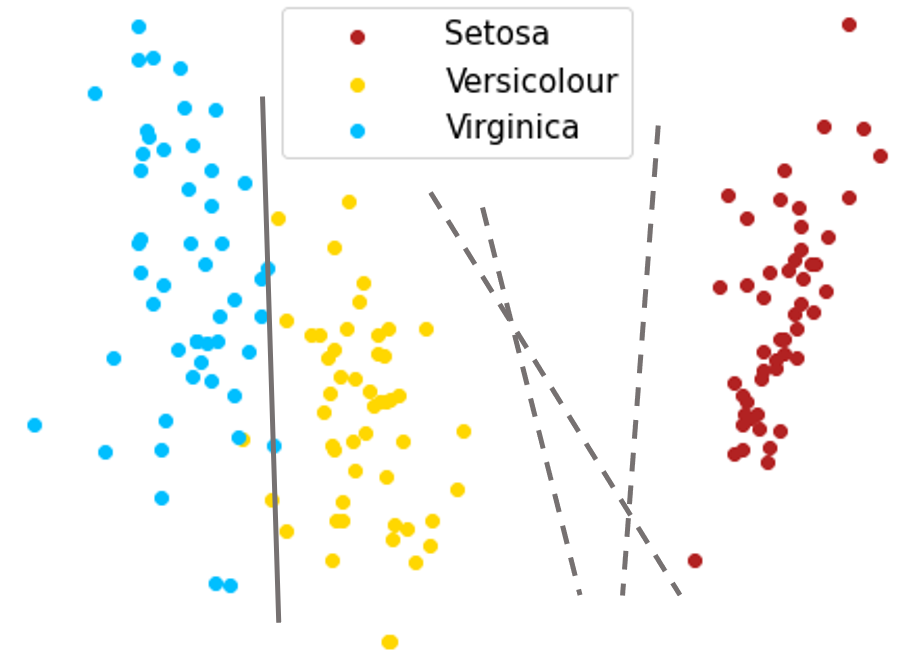} \label{Iris with LDA}}
	\subfloat[Iris with PCA.]{
		\includegraphics[height=2.5cm, width=4.2cm]{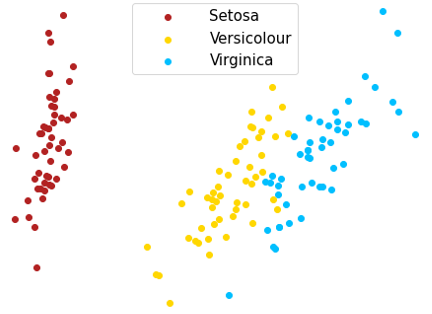} \label{Iris with PCA}}
	\caption{Illustrative figures.}
	\label{Confusion}
\end{figure*}
The two methods above represent unified multi-class models, however, the explicit interpretation of the margin concept is lacking.
Bredensteiner and Bennett \cite{Bredensteiner1999} constructed the M-SVM upon the definition of piecewise-linear separability, which introduced explicit margin for multi-class SVM.
Guermeur \cite{Guermeur2002}
explained multi-class SVM based on uniform strong law of large numbers and proved the equivalence with \cite{Bredensteiner1999}.
The aforementioned methods are representative, and subsequent work mainly revolves around their improvement \cite{Vural2004, Lauer2011, Burg2016}.
The explanation from the margin perspective was collated by Xu et al. \cite{Xu2017}.
Nie el al. proposed capped $\ell_{p}$-norm multi-class SVM to deal with light and heavy outliers \cite{Nie2017}. 
Doan et al. \cite{Doan2016} composed and experimented with the aforementioned multi-class SVM methods.
Lapin et al. \cite{Lapin2015} further proposed a new multi-class SVM model based on a tight convex upper bound of the top-k error. 
Moreover, softmax logical regression \cite{Boehning1992}, the multi-class generalization of logistic regression, is another fundamental classifier in machine learning.
The softmax loss and its variants \cite{Krishnapuram2005, Liu2016, Liu2017} are extensively utilized in the final layer of classification networks.


\section{Method and Discussions}

\subsection{Problem Setup}




Rethinking the decision hyperplane from a plain perspective, with each class associated with parameters $(\mathbf{w},b)$, a desirable linear multi-classifier is supposed to meet
\begin{equation}
	\mathbf{w}_{y_{i}}^{T} \mathbf{x}_{i} + b_{y_{i}}  > \mathbf{w}_{k}^{T} \mathbf{x}_{i} +b_{k}, k \neq y_{i}, \forall i \in [n].
\end{equation}
Accordingly, a $c$-classification problem can be implemented by solving $c$ vectors.
From this, the distinction between samples belonging to class $j$ and class $k$ is established through the decision function:
\begin{equation}
	f_{kl}(\mathbf{x}) = (\mathbf{w}_{k}-\mathbf{w}_{l})^T \mathbf{x} + b_{k} -b_{l}, k<l. 
\end{equation}
Thereby, $f_{kl}(\mathbf{x}) = 0$ is the corresponding separating hyperplane, $f_{kl}(\mathbf{x}) > 0$ for class $k$ and $f_{kl}(\mathbf{x}) < 0$ for class $l$.
The illustration of our multi-class SVM is shown in Figure \ref{Piecewise-linear separator}.


To achieve OvO strategy with a unified model, our multi-class SVM seeks separating hyperplanes by maximizing the margin between two classes such that $(\mathbf{w}_{k} - \mathbf{w}_{l})^T \mathbf{x}_{i}+b_{k} - b_{l}  \geqslant 1$ with $y_{i} = k$ and $(\mathbf{w}_{k} - \mathbf{w}_{l})^T \mathbf{x}_{i}+b_{k} - b_{l} \leqslant -1$ with $y_{i} = l$.
Similar to binary case in Figure \ref{Binary SVM}, the margin of separating between class $k$ and class $l$ is 
\begin{equation}
	Margin\left(C_{k},C_{l}\right) = 2d_{kl} = \frac{2}{\|\mathbf{w}_{k}-\mathbf{w}_{l}\|_{2}}.
\end{equation}
For multi-class SVM, the ideal scheme would be to maximize the margin between class pair. However, due to the mutual limitation between hyperplanes, multi-class SVM cannot do what OvO method does.
Since $c$ vectors are expected to represent $\frac{c(c-1)}{2}$ hyperplanes, there exists mutual restriction between hyperplanes.
Simply put, the hyperplanes that split any three classes are bound to intersect at one point in space, seen in Figure \ref{Piecewise-linear separator}.
This renders the strategy of maximizing the margin between each class pair impractical.

As classification tasks become more complex, researchers are increasingly exploring the semantic similarity between classes \cite{Cacheux2019, Chen2021}.
This arises from the inherent semantic information among classes, resulting in varying classification difficulty,with certain class pairs posing greater challenges to distinguish than others.
For instance, in CIFAR-10 images \cite{Krizhevsky2009}, it is generally more challenging to distinguish cats from dogs than to distinguish birds from trucks.
In handwritten digital images, intuitively, "3" and "5" , "7" and "9" are comparatively similar.
To illustrate this, we perform dimensionality reduction on the widely-used Iris dataset and the results are 
presented in Figure \ref{Iris with LDA} and \ref{Iris with PCA}.
It can be observed that the versicolor and virginica species exhibit higher similarity and pose a greater challenge in practical classification tasks, while the setosa is more distinguishable.
LDA result shows meticulous training is required to obtain a separating plane between versicolor and virginica (solid line), while there is considerable flexibility in separating setosa (dashed line).
Illustrative experiments on other datasets are provided in Appendix A.2.

\subsection{Multi-objective Optimization}
While differences in inter-class similarity are ubiquitous, characterizing them in advance is often challenging.
We address this issue by enlarging the lower bound of the margins,  thereby ensuring the margin between each class pair is not excessively small.
This strategy remains effective in the absence of explicit semantic similarity between classes.
Therefore, the optimization problem characterizing this strategy is
\begin{equation}
	\max_{\mathbf{W} \in \mathbb{R}^{d \times c},\mathbf{b} \in \mathbb{R}^{c}} \min_{k,l \in [c] , k \neq l} Margin\left(C_{k},C_{l}\right)
\end{equation}
Diverging from the previous methods, we consider separating hyperplanes to satisfy linear separability between each class pair.
Hence, the optimization objective can be precisely formulated as follows:
\begin{equation} \label{max_min}
	\begin{aligned}
		&\max_{\mathbf{W} \in \mathbb{R}^{d \times c},\mathbf{b} \in \mathbb{R}^{c}} \min_{k,l \in [c] , k < l}   \frac{1}{ \|\mathbf{w}_{k}-\mathbf{w}_{l}\|_{2}},\\
		{ s.t. }& \begin{cases}
			f_{kl}(\mathbf{x}_{i}) \geqslant 1, &y_{i} = k,\\ f_{kl}(\mathbf{x}_{i}) \leqslant -1, &y_{i} = l,
		\end{cases} \quad  i \in [n].
	\end{aligned}
\end{equation}
It should be pointed out that the lower bound optimization strategy is not sufficient to focus on the all classes, which may lead to nearly equal margins. We next convert problem (\ref{max_min}) into a tractable form with the following lemma.
\begin{lemma}\label{equ_p}
	Assume $ g_{1}(\mathbf{z}),g_{2}(\mathbf{z}),\cdots,g_{m}(\mathbf{z})$ are given real functions.
	The following two optimization problem is equivalent when $p\rightarrow \infty$:
	\begin{equation}
		(a) \max_{\mathbf{z}} \min_{j} g_{j}(\mathbf{z}) \qquad \qquad (b) \min_{\mathbf{z}} \sum_{j = 1}^{m} g^{-p}_{j}(\mathbf{z}).
	\end{equation}
\end{lemma}
This lemma can be interpreted as an objective of multitasking and leads to the following theorem.

\begin{theorem}
	The problem (\ref{max_min}) is equivalent to
	\begin{equation}\label{hard_p}
	\begin{aligned}
		&\min_{\mathbf{W} \in \mathbb{R}^{d \times c},\mathbf{b} \in \mathbb{R}^{c}} \sum_{k=1}^{c-1} \sum_{l=k+1}^{c} \|\mathbf{w}_{k} - \mathbf{w}_{l}\|_{2}^{p},\\
		&{ s.t. } \begin{cases}
			f_{kl}(\mathbf{x}_{i})  \geqslant 1, &y_{i} = k,\\ f_{kl}(\mathbf{x}_{i})  \leqslant -1, &y_{i} = l,
		\end{cases} \quad i \in [n].
	\end{aligned}
\end{equation}
	with the given parameter $p \rightarrow \infty$.
\end{theorem}

We introduce $p$ into the model as a tunable parameter so that problem (\ref{hard_p}) is not only an approximation for problem (\ref{max_min}), but also takes into account the margins between all class pairs.
An appropriate $p$ is supposed to globally enlarge the margin between each class pair while enhancing the lower bound of the margin.
The tunable range of $p$ encompasses positive real numbers.

Given the ubiquitous prevalence of inseparable instances, slack variable between each class pair can be introduced:
\begin{equation}\label{initial_slack}
	\begin{aligned}
		&\min_{\mathbf{W}, \mathbf{b}}    \sum_{k<l} \sum_{y_{i}\in \{k,l\}} \xi_{ikl} + \lambda \sum_{k<l} \|\mathbf{w}_{k}-\mathbf{w}_{l}\|_{2}^{p},\\
		&{ s.t. } \begin{cases}
			f_{kl}(\mathbf{x}_{i})  \geqslant 1 - \xi_{ikl}, &y_{i} = k,\\
			f_{kl}(\mathbf{x}_{i})  \leqslant -1 + \xi_{ikl}, &y_{i} = l,
		\end{cases}  \quad i \in [n].
	\end{aligned}
\end{equation}
By substituting the slack variables, Eq. (\ref{slack}) can be transformed into an unconstrained optimization problem:
\begin{equation}\label{initial}
		\min_{\mathbf{W}, \mathbf{b}}  \sum_{k<l} \sum_{y_{i}\in \{k,l\}} [1-y_{ikl}f_{kl}(\mathbf{x}_{i}) ]_{+} + \lambda \sum_{k<l} \|\mathbf{w}_{k}-\mathbf{w}_{l}\|_{2}^{p}
\end{equation}
where
\begin{equation}
	y_{ikl} =
	\begin{cases}
		1, & \text{if } y_i = k, \\
		-1, & \text{if } y_i = l.
	\end{cases}
\end{equation}

\subsection{Unique Solution}
Assume $\mathbf{W}$ and $\mathbf{b}$ are the optimal solution of problem (\ref{initial}).	
It is obvious that, for an arbitrary $\bm{\sigma} \in \mathbb{R}^{d}$,  suppose $\mathbf{\tilde{W}} = \mathbf{W} + \bm{\sigma} \mathbf{1}^{T}$, there is $\mathbf{w}_{j}-\mathbf{w}_{k} = \mathbf{\tilde{w}}_{j}-\mathbf{\tilde{w}}_{k}$.
For an arbitrary $\eta \in \mathbb{R}$, suppose $\mathbf{\tilde{b}} = \mathbf{b} + \eta \mathbf{1}$, there is $ \mathbf{b}_{j}-\mathbf{b}_{k} = \mathbf{\tilde{b}}_{j}-\mathbf{\tilde{b}}_{k}$.
Therefore $\mathbf{\tilde{W}}$ and $\mathbf{\tilde{b}}$ is also the optimal solution.
It is expected that the model has a unique optimal solution for a specific dataset that is not contingent upon the initialization.
Without loss of generality, we impose a mean-zero constraint on $\mathbf{W}$ and $\mathbf{b}$ that $\mathbf{W} \mathbf{1} = \mathbf{0}$,$\mathbf{b}^{T} \mathbf{1} = 0$.
Here the constraint is converted into an readily solvable form by the following theorem.

\begin{theorem}\label{th_constraints}
	Assume function $f(\mathbf{Z})$ has the property of column translation invariance, i.e., $\forall \bm{\sigma} \in \mathbb{R}^{n}$, there is $f(\mathbf{Z}) = f(\mathbf{Z} + \bm{\sigma} \mathbf{1}^{T}) $.
	With given $\varepsilon \rightarrow 0$, the following two optimization problems have the same optimal solution
	\begin{equation}\label{column}
		\min_{\mathbf{Z} \in \mathbb{R}^{n \times m}} f(\mathbf{Z}), \quad { s.t. } \sum_{j = 1}^{m} \mathbf{z}_{j} = \mathbf{0},
	\end{equation}
	\begin{equation}\label{penalty_Z}
		\min_{\mathbf{Z} \in \mathbb{R}^{n \times m}} f(\mathbf{Z}) + \varepsilon \|\mathbf{Z}\|_{F}^{2}.
	\end{equation}
\end{theorem}
Considering the above theorem, the mean-zero constraint can be transformed into an additional penalty term. The complete model can be written as 
\begin{equation} \label{obj}
	\begin{aligned}
		&\min_{\mathbf{W}  \in \mathbb{R}^{d \times c}, \mathbf{b} \in \mathbb{R}^c} \sum_{k<l} \sum_{y_{i}\in \{k,l\}} [1-y_{ikl}f_{kl}(\mathbf{x}_{i}) ]_{+} + \\
		&\qquad \qquad \lambda \sum_{k<l} \|\mathbf{w}_{k}-\mathbf{w}_{l}\|_{2}^{p} + \varepsilon \left(\|\mathbf{W}\|_{F}^{2} + \|\mathbf{b}\|_{2}^{2}\right).
	\end{aligned}
\end{equation}
$\lambda$ is a trade-off parameter to balance error tolerance and enlargment of inter-class margin.
For a small $\lambda$, instances that fall within the margin receive a high penalty, whereas for a larger $\lambda$, the penalty  decreases.
The effect of $\varepsilon$ to the model is negligible, just for the purpose of unique solution.
$\varepsilon$ is fixed to $1 \times 10^{-6}$ in the experiments.

Since the computation of inter-class loss is expensive in practice, we simplify it by the following theorem. 
	\begin{theorem}\label{th_Sigma}
		The following equation holds,
		\begin{equation}
			\sum_{k<l} \sum_{y_{i}\in \{k,l\}} \left[1-y_{ikl} f_{kl}(\mathbf{x}_{i}) \right]_{+} = \sum_{i=1}^{n} \sum_{k \neq y_{i}}	[1 - f_{y_{i}k}(\mathbf{x}_{i}) ]_{+}.
		\end{equation}
	\end{theorem}
	According to Theorem \ref{th_Sigma}, problem (\ref{obj}) is equivalent to
	\begin{equation} \label{i_obj}
		\begin{aligned}
			&\min_{\mathbf{W}  \in \mathbb{R}^{d \times c}, \mathbf{b} \in \mathbb{R}^c} \sum_{i=1}^{n} \sum_{k \neq y_{i}}	[1 - f_{y_{i}k}(\mathbf{x}_{i}) ]_{+} +  \\
			&\qquad \qquad \lambda \sum_{k<l} \|\mathbf{w}_{k}-\mathbf{w}_{l}\|_{2}^{p} + \varepsilon \left(\|\mathbf{W}\|_{F}^{2} + \|\mathbf{b}\|_{2}^{2}\right).
		\end{aligned}
	\end{equation}
	In this way, we transform our loss into the summation of the individual sample losses during optimization.
	
\subsection{Smoothness and Convexity}	
	Our method is intended to be applicable to gradient optimization, so the non-smooth hinge loss needs to be dealt with.
	We employ the following straightforward function to approximate $[x]_{+}$:
	\begin{equation}\label{smooth}
		g(x) = \frac{x + \sqrt{x^2 + \delta^2}}{2},\quad (\delta > 0).
	\end{equation}
	\begin{lemma}
		$g(x)$ satisfies $0 \leqslant g(x) - [x]_{+}  \leqslant \frac{\delta}{2}$.
		When $\delta \rightarrow 0$, $g(x) \rightarrow [x]_{+}$.
	\end{lemma}

	The closeness between two functions exclusively depends on the proximity factor $\delta$.
By replacing the hinge loss, our overall model can be formulated as
	\begin{equation} \label{smooth_obj}
		\begin{aligned}
			&\min_{\mathbf{W}  \in \mathbb{R}^{d \times c}, \mathbf{b} \in \mathbb{R}^c} \sum_{i=1}^{n} \sum_{k \neq y_{i}}	\frac{\gamma_{ik} + \sqrt{\gamma_{ik}^2 + \delta^2}}{2} +  \\
			&\qquad \qquad \lambda \sum_{k<l} \|\mathbf{w}_{k}-\mathbf{w}_{l}\|_{2}^{p} + \varepsilon \left(\|\mathbf{W}\|_{F}^{2} + \|\mathbf{b}\|_{2}^{2}\right).
		\end{aligned}
	\end{equation}
	where $\gamma_{ik} = 1-f_{y_{i}k}(\mathbf{x}_{i})$.
The decision function between class $k$ and class $l$ is $f_{kl}(\mathbf{x}_{i}) = (\mathbf{w}_{k}-\mathbf{w}_{l})^T \mathbf{x}_i+b_{k} -b_{l}$.
It is worth noting that the alterations of $\varepsilon$ and $\delta$ wield negligible influence on the model.
	The only parameters that affect the model performance are actually $p$ and $\lambda$.
	
\begin{theorem}
	Problem (\ref{smooth_obj}) is strictly convex.
\end{theorem}

Therefore, we adopt Adam \cite{Kingma2014} optimization strategy to update $\mathbf{W}$ and $\mathbf{b}$ in problem (\ref{smooth_obj}), which enables our method to be applied to the deep model.
For the sample $(\mathbf{x}_{i} ,y_{i})$ in the train set, there are two cases when we take derivative of the objective function Eq. (\ref{smooth_obj}) with respective to $\mathbf{W}$ by column.
If $k = y_{i}$ , the derivative with respect to $\mathbf{w}_{k}$ at iteration $t$ is
\begin{equation}\label{case1}
	\begin{aligned}
		\nabla_{k}^{(t)} = & -  \sum_{l \neq k}  \frac{\gamma_{il} + \sqrt{\gamma_{il}^2 + \delta^2}}{2 \sqrt{\gamma_{il}^2 + \delta^2}} \mathbf{x}_{i} + 2 \varepsilon \mathbf{w}_{k} + \\
		&\sum_{l \neq k}      \lambda p\|\mathbf{w}_{k}-\mathbf{w}_{l}\|_{2}^{p-2}  \left(\mathbf{w}_{k}-\mathbf{w}_l\right).
	\end{aligned}
\end{equation}

If $k \neq y_{i}$ , the derivative goes to:
	\begin{equation}\label{case2}
		\begin{aligned}
			\nabla_{k}^{(t)} = &\frac{\gamma_{ik} + \sqrt{\gamma_{ik}^2 + \delta^2}}{2 \sqrt{\gamma_{ik}^2 + \delta^2}} \mathbf{x}_{i} + 2 \varepsilon \mathbf{w}_k +  \\
			&  \sum_{l \neq k}      \lambda p\|\mathbf{w}_{k}-\mathbf{w}_{l}\|_{2}^{p-2}  \left(\mathbf{w}_k-\mathbf{w}_l\right).
		\end{aligned}
	\end{equation}
	Note that the optimization is performed for $\mathbf{W}$ as a whole, independent of the order in which the columns are updated.
The convexity of the objective function ensures its convergence to the global optimum through the employment of Adam.
After the training process, test sample $\mathbf{x}$ can be classified through $y = \emph{arg} \max_{k} \mathbf{w}_{k}^T \mathbf{x}+b_{k}$.

\subsection{Connection to $\ell_{2}$-regularizer}

One might consider replacing the regularizer with  the form of $\ell_{2}$-norm sum: $\sum_{k=1}^{c} \|\mathbf{w}_{k}\|^{2}_{2}$, like the regularization of most classifiers \cite{Zhang2003, Raman2019}.
The relationship between the two is summarized as follows.
\begin{lemma}\label{variance}
	The following equation holds,
	\begin{equation}
		\sum_{k=1}^{c-1} \sum_{l=k+1}^{c} \|\mathbf{w}_{k} - \mathbf{w}_{l}\|_{2}^{2} = c \sum_{k=1}^{c} \|\mathbf{w}_{k} - \frac{1}{c} \sum_{l=1}^{c}\mathbf{w}_{l}\|_{2}^{2}.
	\end{equation}
\end{lemma}

According to Theorem \ref{th_constraints}, there is $\sum_{l=1}^{c}\mathbf{w}_{l} = \mathbf{0}$ when $\mathbf{W}$ is taken to be optimal in problem (\ref{smooth_obj}).
So we draw the conclusion that with $p = 2$, problem (\ref{smooth_obj}) and the problem replacing the regular term  with $\ell_{2}$ regularizer have the same optimal solution.
\begin{theorem}\label{theorem p=2}
	The optimization problem (\ref{i_obj}) with $p=2$ has the same optimal solution as
	\begin{equation} \label{sq_obj}
		\min_{\mathbf{W}, \mathbf{b} } \sum_{i=1}^{n} \sum_{k \neq y_{i}}	[1 - f_{y_{i}k}(\mathbf{x}_{i}) ]_{+} +   \lambda c \sum_{k=1}^c\left\|\mathbf{w}_{k}\right\|_2^2  + \varepsilon \|\mathbf{b}\|_{2}^{2}.
\end{equation}
\end{theorem}
According to Theorem \ref{theorem p=2}, Eq. (\ref{Weston}) can be regarded as a special case of our method when $p = 2$.
The only difference is that Eq. (\ref{Weston}) has no constraint on $\mathbf{b}$, which leads to non-unique solutions.
However, the naive $\ell_{2}$ regularizer does not have a clear geometric meaning.
As will be demonstrated in the experiments, $p = 2$ is not optimal in a wide range of cases.
In addition, $\ell_{1}$ regularizer is regarded as an effective way to preform simultaneous variable selection and classification \cite{Bradley1998, Guyon2003}.
It should be pointed out that due to the superiority of the $\ell_{1}$ norm in high-dimensional case (gene, document), our proposed regularizer can be directly replaced by the $\ell_{1}$ regularizer $\sum_{k<l} \|\mathbf{w}_{k} - \mathbf{w}_{l}\|_{1}^{p}$, which will not be expanded in detail here.

\renewcommand\arraystretch{1.7}
\begin{table*}[t]\large
	\vspace{-2mm}
	\begin{center}
		\caption{Average performance (w.r.t.  Accuracy) on test set over $10$ runs by different methods.}
		\label{table:Accuracy}
		\resizebox{\linewidth}{26mm}{
			\begin{tabular}{ccccccccc}
				\toprule[2pt]
				Methods   & OvR          & OvO           & Crammer          & M-SVM          & Top-k                & Multi-LR	& SMLR                   & M$^3$SVM                  \\ \midrule
				\slshape{Cornell}      & 0.812 $\pm$ 0.065 & 0.845 $\pm$ 0.028 & 0.792 $\pm$ 0.015 & 0.755 $\pm$ 0.031 & 0.826 $\pm$ 0.016       & 0.783 $\pm$ 0.026  	& 0.803 $\pm$ 0.009        & \textbf{0.865 $\pm$ 0.013}  \\
				\slshape{ISOLET}        & 0.866 $\pm$ 0.046 & 0.942 $\pm$ 0.004 & 0.922 $\pm$ 0.042 & 0.910 $\pm$ 0.004 & 0.904 $\pm$ 0.013     & 0.940 $\pm$ 0.004  & 0.926 $\pm$ 0.008  & \textbf{0.945 $\pm$ 0.002} \\
				\slshape{HHAR}      & 0.845 $\pm$ 0.059 & 0.966 $\pm$ 0.014 & 0.931 $\pm$ 0.039 & 0.953 $\pm$ 0.008 & 0.970 $\pm$ 0.007       & 0.948 $\pm$ 0.010    & 0.952 $\pm$ 0.012      & \textbf{0.981 $\pm$ 0.004} \\
				\slshape{USPS}       &0.887 $\pm$ 0.042 & 0.898 $\pm$ 0.005 & 0.769 $\pm$ 0.047 & 0.910 $\pm$ 0.018 & 0.825 $\pm$ 0.009 & 0.932 $\pm$ 0.002    & 0.937 $\pm$ 0.004      & \textbf{0.956 $\pm$ 0.011} \\
				\slshape{ORL}   & 0.919 $\pm$ 0.021 & \textbf{0.975 $\pm$ 0.000} & 0.879 $\pm$ 0.018 & 0.790 $\pm$ 0.034 & 0.879 $\pm$ 0.028 & 0.925 $\pm$ 0.000 & 0.925 $\pm$ 0.000 &\textbf{0.975 $\pm$ 0.000}   \\
				\slshape{Dermatology} & 0.939 $\pm$ 0.009 & 0.971 $\pm$ 0.003 & 0.933 $\pm$ 0.015 & 0.868 $\pm$ 0.031 & 0.891 $\pm$ 0.047       & 0.965 $\pm$ 0.007    & 0.965 $\pm$ 0.010      & \textbf{0.988 $\pm$ 0.001} \\
				\slshape{Vehicle}       & 0.794 $\pm$ 0.016 & 0.756 $\pm$ 0.024 & 0.757 $\pm$ 0.021 & 0.762 $\pm$ 0.019 & 0.778 $\pm$ 0.007       & 0.780 $\pm$ 0.010 & 0.771 $\pm$ 0.020 & \textbf{0.800 $\pm$ 0.011}   \\
				\slshape{Glass}        & 0.656 $\pm$ 0.075 & 0.685 $\pm$ 0.008 & 0.594 $\pm$ 0.045 & 0.629 $\pm$ 0.044 & 0.674 $\pm$ 0.025       & 0.664 $\pm$ 0.018  & 0.679 $\pm$ 0.015 & \textbf{0.744 $\pm$ 0.007}  \\ \midrule
			\end{tabular}
		}
	\end{center}
\vspace{-2mm}
\end{table*}

\subsection{Structural Risk Minimization}
Subsequently, we  elucidate the explanation of the prposed M$^3$SVM from structural risk minimization (SRM) inductive principle.
Following the fundamental assumptions in statistical learning theory, there is a canonical but unknown joint distribution on $\mathcal{X} \times \mathcal{C}$.
The goal of learning is to select a function $f: \mathbf{x} \rightarrow \mathbb{R}^{c}$ (or in terms of probability $f: \mathbf{x} \rightarrow [0,1]^{c}$),  among from a specific design functions space   $\mathcal{F}$, such that its error on the joint distribution is minimized.
The discriminant function for the classification problem is typically in the form of $g(\mathbf{x}) = \max_{j} f_{j}(\mathbf{x})$.
The risk of the classification task can be written as
$
	\mathcal{R}(f)  = \int \mathbb{I}(g(\mathbf{x}) \neq y) d P(\mathbf{x},y).
$
Since $P(\mathbf{x},y)$ is unknown, one shallow solution is to minimize empirical risk
$	\mathcal{R}_{e} (f) = \frac{1}{n} \sum_{i}^{n} \mathbb{I}(g(\mathbf{x}) \neq y)$ on certain samples.
SRM inductive principle is a more recognized technique, which is based on the theory that for any $f \in \mathcal{F}$ with a probability of at least $1-\rho$, the risk meets
$
	\mathcal{R} \leqslant \mathcal{R}_{e} + \Omega(\mathcal{F},\rho,n),
$
where $\Omega$ is called guaranteed risk and can be expressed in the form of VC dimension \cite{Vapnik2015}, Rademacher complexity \cite{Bartlett2002}, etc.
The learnable basis of the binary SVM \cite{Schiilkop1995} is to reduce the risk of VC dimensional form by minimizing $\|\mathbf{w}\|_{2}$.
For multi-class SVM model, $f$ can be set as a multi-valued function $f:\mathbf{x} \rightarrow \mathbf{w}_{k} \mathbf{x} + b_{k},k \in [c].$
Note that different $f$ in $\mathcal{F}$ are only differ in the parameters $\mathbf{W}$ and $\mathbf{b}$ in this case.

The theory of generalized risk derived from Uniform Strong Law of Large Numbers \cite{Peres2008, Guermeur2002}  implement the SRM
inductive principle by delineating a  compromise between training performance and complexity.
It can be briefly summarized as follows.
For multi-class SVM model, minimizing its guaranteed risk can be approximately equated to minimizing a norm of the linear operator $\|T(f)\|_{\omega}$, where functional $T: \mathcal{F} \rightarrow \mathbf{M}_{2n \times c(c-1)/2}$ mapping a function to a real matrix.
The norm  is chosen in accordance with the choice of the pseudo-metric on $\mathcal{F}$, for instance, $\forall(f, \bar{f}) \in \mathcal{F}^2$,
\begin{equation}
	\begin{aligned}
		\omega_{l_{\infty}, l_1}(f, \bar{f})&=\max _{x} \sum_{k<l}\left|f_k(x)-\bar{f}_k(x)\right|,\\
		\omega_{l_{\infty}, l_{\infty}}(f, \bar{f})&=\max _{x} \max_{k < l} \left|f_k(x)-\bar{f}_k(x)\right|,
	\end{aligned}
\end{equation}
which correspond to matrix norm $\|\mathbf{M}\|_{l_{\infty}, l_1}$ and $\|\mathbf{M}\|_{l_{\infty}, l_{\infty}}$ respectively.
Define $T(f) = [t^{(1)}, \cdots, t^{(2n)}]^{T}$, where $t^{(i)}(f) = [(\mathbf{w}_1 - \mathbf{w}_2)^{T} \mathbf{x}_{i}, \cdots,(\mathbf{w}_{k} -\mathbf{w}_{l})^{T} \mathbf{x}_{i}, \cdots, (\mathbf{w}_{c-1} -\mathbf{w}_{c})^{T} \mathbf{x}_{i}]^{T} \in \mathbf{M}_{1 \times c(c-1)/2}$.
The matrix norm of $T(f)$ provides a tight upper bound for the guaranteed risk $\Omega$.
In M$^3$SVM, the infimum of margin bears a close relationship with the crude upper bound of  the norm of $T(f)$.
It can be put down to the following theorem.

\begin{theorem}
	Let $\mathcal{F}$ be the multivariate linear model from $\mathcal{X}$ into $\mathbb{R}^{c}$.
	$\mathcal{F}$ are endowed  with the Euclidean norm.
	If $\mathcal{X}$ is included in a ball of radius $\Lambda_{\mathcal{X}}$ about the origin, then $\forall f \in \mathcal{F}$(parametrized by $\mathbf{W}$ and $\mathbf{b}$) the following bound holds:
	\begin{equation}
		\|T(f)\|_{l_{\infty}, l_{p}} \leqslant  \Lambda_{\mathcal{X}} \left(\sum_{k<l}\|\mathbf{w}_{k} - \mathbf{w}_{l}\|_{2}^{p}\right)^{\frac{1}{p}}.
	\end{equation}
\end{theorem}

Our regularizer optimizes an upper bound on the guaranteed risk derived from covering numbers \cite{Guermeur2002}.
Combining the methods described in the previous section, when $p\rightarrow\infty$, minimizing $\sum_{k=1}^{c-1} \sum_{l=k+1}^{c} \|\mathbf{w}_{k} - \mathbf{w}_{l}\|_{2}^{p}$ is equivalent to maximizing $\inf_{k<l} Margin\left(C_{k},C_{l}\right)$, which reduces the upper bound of guaranteed risk.
Moreover, this method is essentially adapted to multiple metrics of $\mathcal{F}$, whereas the previous methods correspond to a special case when $p = 2$.
We draw the conclusion that our method is interpretable in terms of SRM and it is altogether possible to improve the generalization performance (i.e. reduce the guaranteed risk $\Omega$) by maximizing the minimum margin.

\subsection{Extension to Softmax Loss}
Our proposed method exhibits versatility, extending its applicability to other linear classifiers, such as logistic regression (LR).
By altering the misclassification loss, our method acts as a regularized softmax loss and can be applied to the last layer of the neural network.
As it can learn embeddings with large inter-class margins, the proposed loss guides the learning of network parameters through backpropagation. 
Geometric interpretation and discussions are relegated to Appendix A.3.

\section{Experiments}
In this section, we empirically evaluate the effectiveness of our method on multi-class classification task and analyze the experimental results.

\subsection{Experiment Settings}
The datasets chosen for evaluation include Cornell, ISOLET, HHAR, USPS, ORL, Dermatology, Vehicle and Glass, which represent diverse data types (including image, speech, document, etc).
They can all be found at  \footnote{\url{https://archive.ics.uci.edu/ml/datasets.php}}.
The details of the datasets are described in Appendix A.4.
Our method is compared with six linear multi-classification methods, including OvR\cite{Vapnik1998}, OvO \cite{Hsu2002}, Crammer \cite{Crammer2001}, M-SVM \cite{Bredensteiner1999}, Top-k \cite{Lapin2015},
Multi-LR \cite{Boehning1992} and Sparse Multinomial Logistic Regression \cite{Krishnapuram2005} (SMLR).
These methods have been mentioned in the previous sections.
For the sake of fairness and generalizability, all methods are directly trained in the original feature space rather than in well-selected kernel spaces.
For the hyperparameters involved in M$^3$SVM, $\lambda$ is set to ten equidistant values within the interval $[1\times 10^{-4},1\times 10^{-1}]$, while $p$ is set on a grid of $[1, 2, \cdots, 8]$.
For all comparative methods, we adhere to the authors' default parameter settings and, where necessary, similarly conduct parameter grid searches to achieve fair comparisons as far as possible.
We evaluate the performance of each method with five-fold cross-validation and report the average results of $10$ runs.

\subsection{Results}
The experimental results of M$^3$SVM and seven comparative methods are reported in Table \ref{table:Accuracy}, where each result represents the average test accuracy and standard deviation of $10$ runs.
The best result on each dataset is marked in bold.
In comparison to other widely used multi-classification algorithms,  M$^3$SVM achieves the best classification performance on all selected datasets, which can be attributed to the flexible factor $p$ that adapt to diverse data structures.
It is noteworthy that the OvO method generally exhibits great performance, owing to its independent separation of any two classes.
Unfortunately, the complexity of classification for new arrival data considerably limits its application.
In addition to this, the tuning parameter $\lambda$ for each subproblem in OvO are not straightforward, which accounts for the suboptimal experimental results.

Beyond the convexity, the sound convergence property is experimentally verified.
The variation of the objective function values and the accuracy (ACC) on test set over the number of iterations is depicted in Figure \ref{conver} on six datasets.
Throughout the entire training process, accuracy consistently improves as the loss function decreases. 
It can be found that M$^3$SVM converges rapidly, with convergence observed within $500$ iterations across all six datasets.
\begin{figure}[h]
	\centering  
	\subfloat[ISOLET]{
		\includegraphics[height=2cm, width=2.7cm]{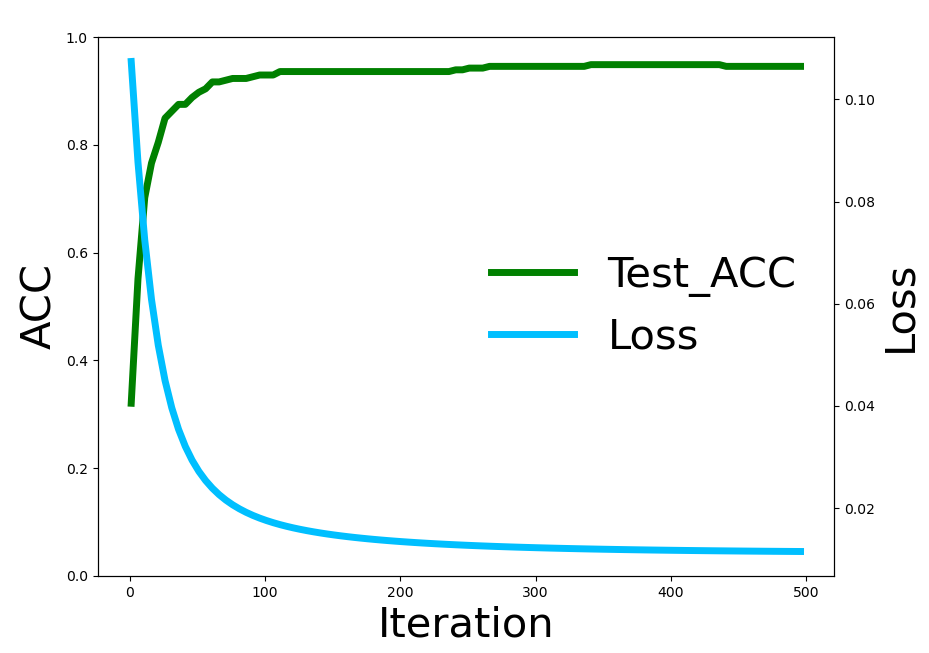}}
	\subfloat[HHAR]{
		\includegraphics[height=2cm, width=2.7cm]{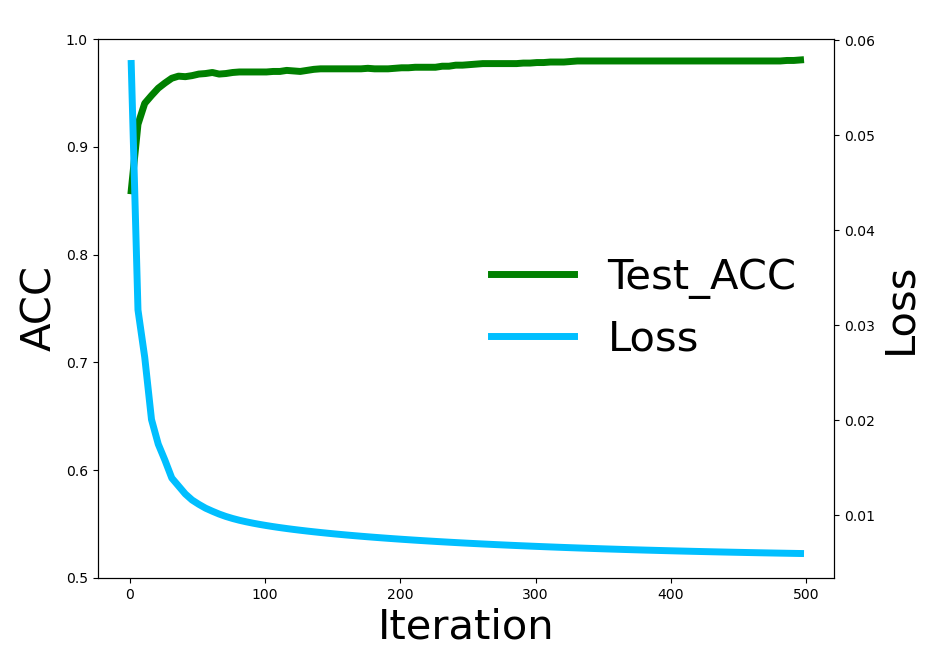}}
	\subfloat[Cornell]{
		\includegraphics[height=2cm, width=2.7cm]{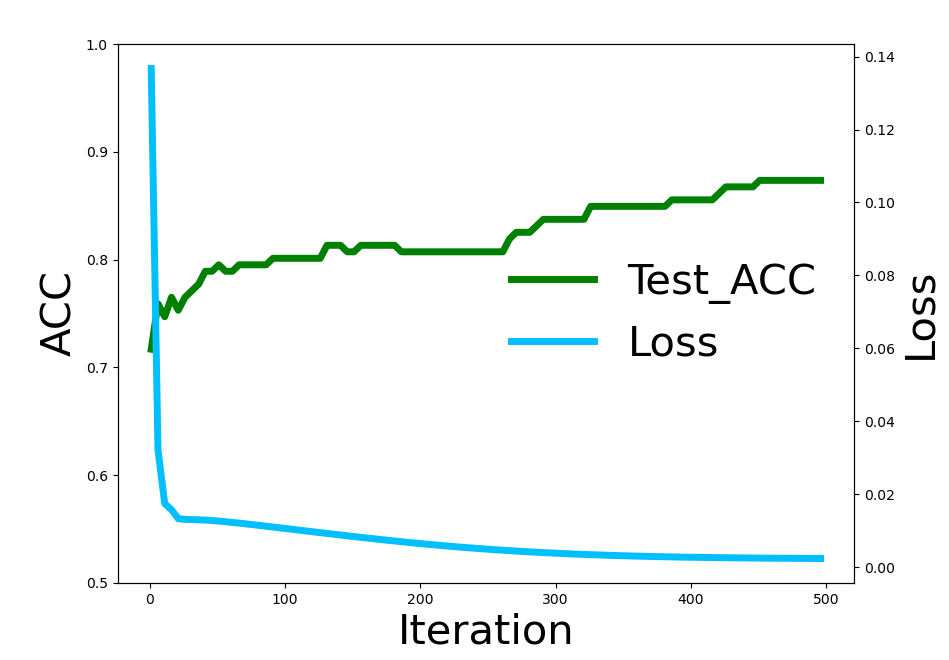}}
	\\
	\subfloat[USPS]{
		\includegraphics[height=2cm, width=2.7cm]{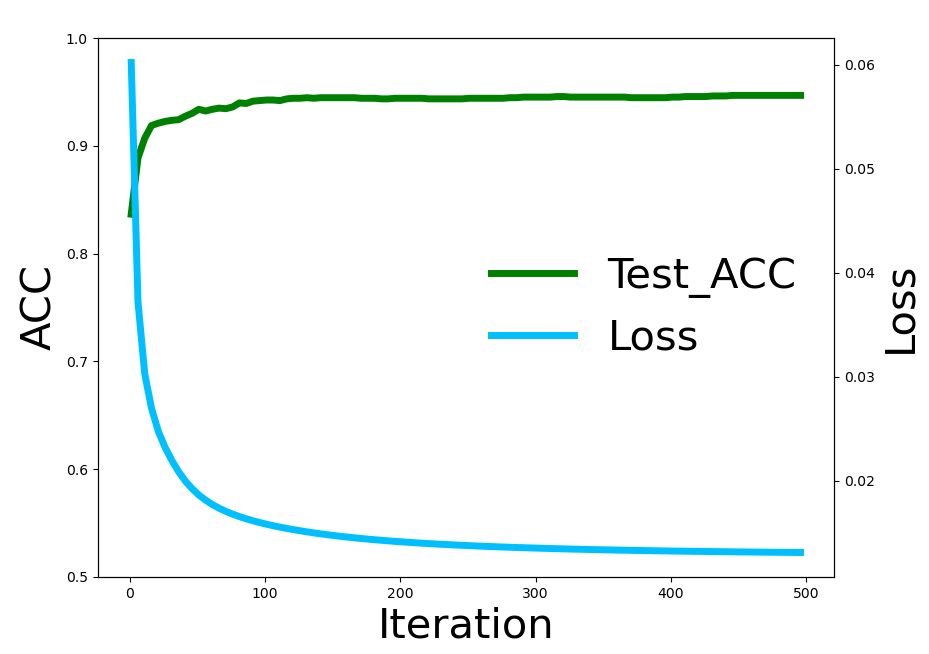}}
	\subfloat[Glass]{
		\includegraphics[height=2cm, width=2.7cm]{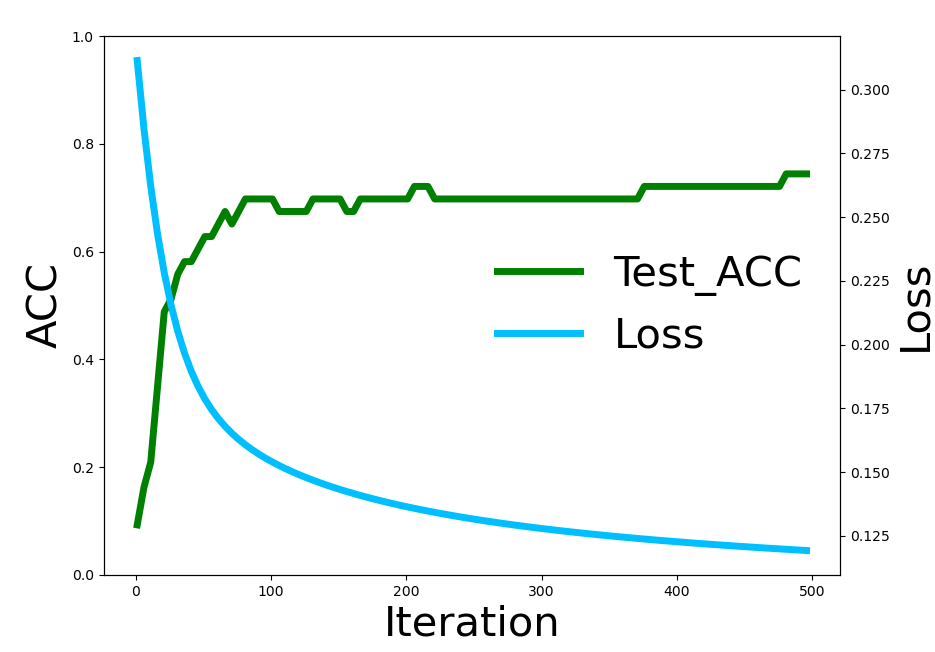}} 
	\subfloat[Vehicle]{
		\includegraphics[height=2cm, width=2.7cm]{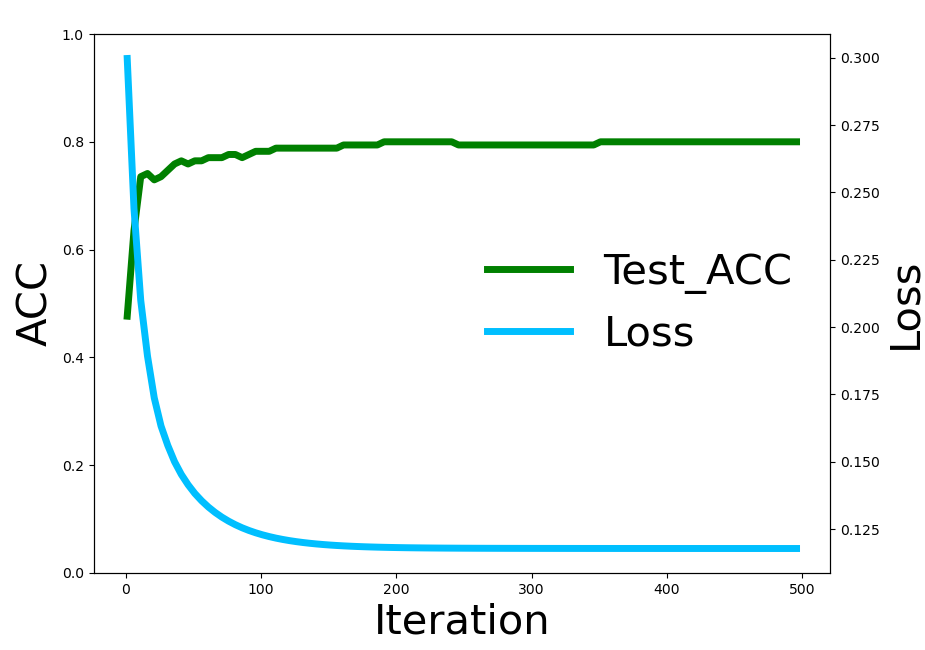}}
	\caption{Convergence of the objective function value.} \label{conver}
\end{figure}

The test accuracy under various values of $p$ on six datasets is presented in Figure \ref{p-value}. 
A noteworthy  observation is that as the value of $p$ increases, the model's generalization performance (test accuracy) initially peaks and subsequently diminishes (note that it is totally possible that the peak is not within $[1,8]$).
This phenomenon aligns with our motivation, where $p$ acts as a balancing factor between the global margins and the lower bound of the margins.
The increase of $p$ can be interpreted as a prioritization of enhancing the lower bound of the margin.
When $p$ is small, enlarging the margin between each class pair contributes to the reduction of the objective function.
When $p$ is large, the objective function primarily emphasizes boosting the lower bound of margins.
In such case, the value of the rest margins are pulled down due to the interlocking separating hyperplanes.
Through extensive experiments, we found model performance is generally better when $p$ is around $4$, which can serve as a reasonable prior.
Furthermore, it's advisable to avoid excessively large values of $p$, as they may result in poor convergence.
\begin{figure}[h]
	\centering  
	\subfloat[ISOLET]{
		\includegraphics[height=1.8cm, width=2.6cm]{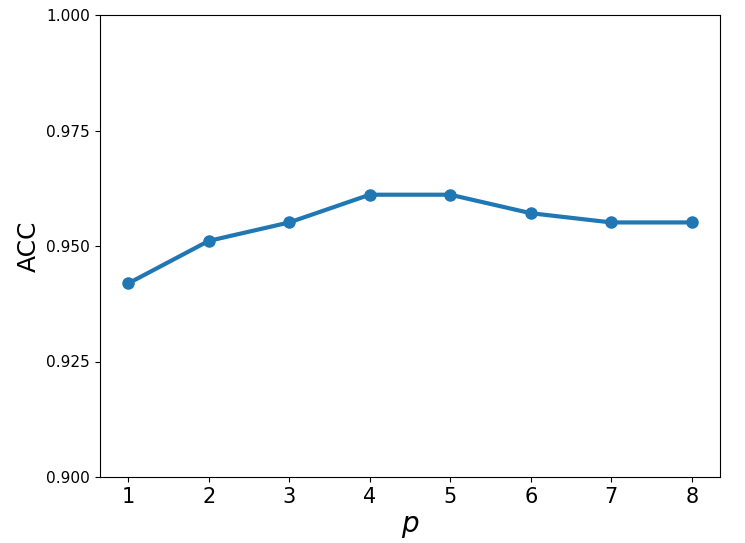}}
	\subfloat[HHAR]{
		\includegraphics[height=1.8cm, width=2.6cm]{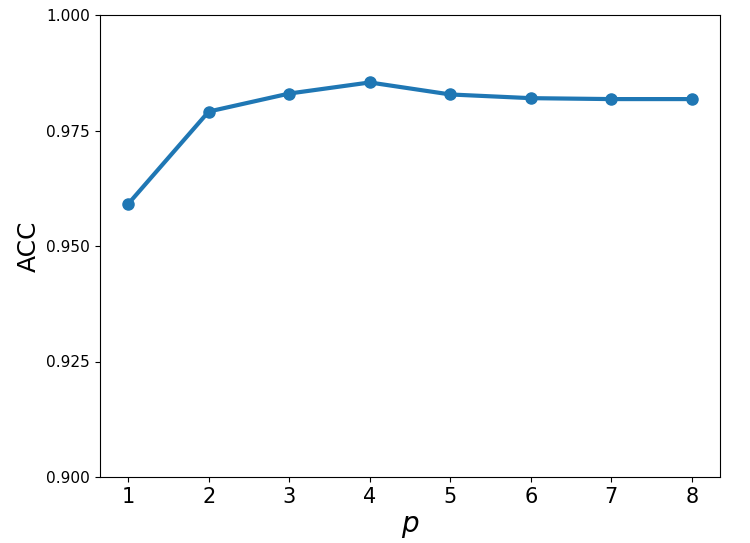}}
	\subfloat[Cornell]{
		\includegraphics[height=1.8cm, width=2.6cm]{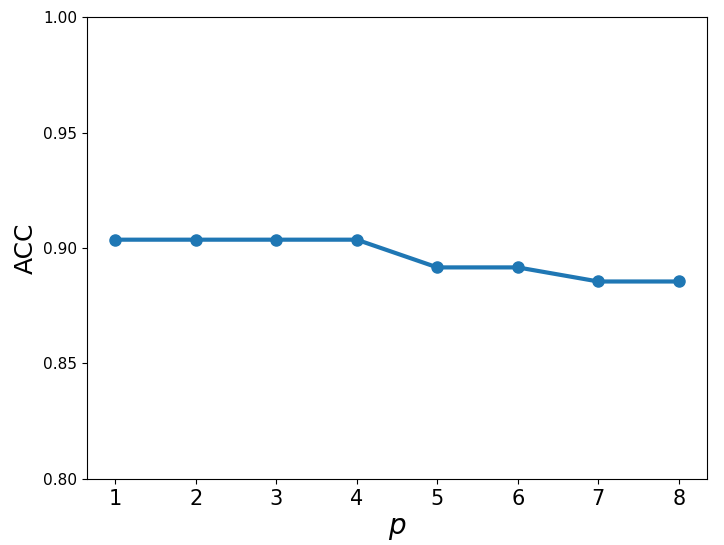}}
	\\
	\subfloat[USPS]{
		\includegraphics[height=1.8cm, width=2.6cm]{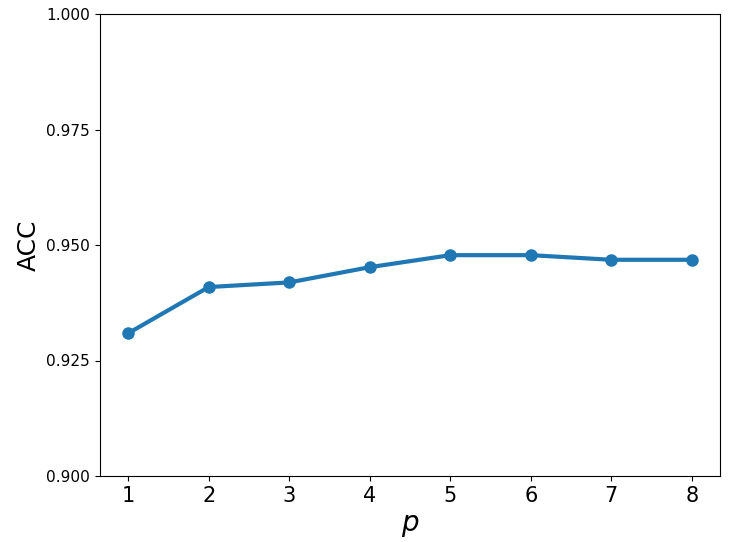}}
	\subfloat[Glass]{
		\includegraphics[height=1.8cm, width=2.6cm]{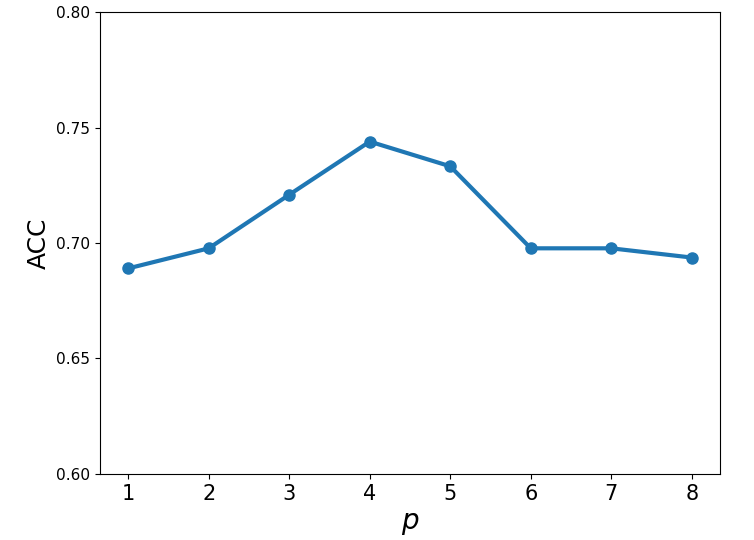}} 
	\subfloat[Vehicle]{
		\includegraphics[height=1.8cm, width=2.6cm]{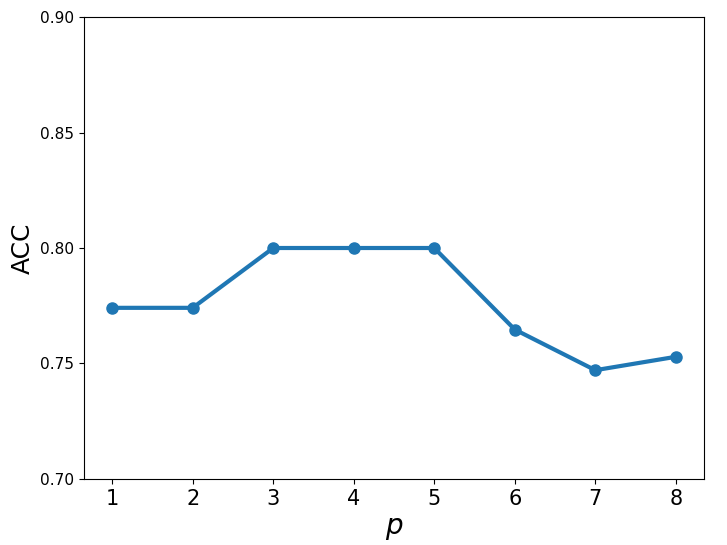}}
	\caption{The effect of parameter $p$ on experimental results.} \label{p-value}
\end{figure}

We study the sensitivity of the trade-off parameter $\lambda$.
Figure \ref{ana_lambda} illustrates the variation of the test accuracy on the eight datasets at $p=4$ as a function of $\lambda$, within the range of $[1\times 10^{-4},1]$.
One can infer that $\lambda$ ensures the generalization performance over a broad range.
Empirically, the margin term is typically several orders of magnitude larger than the loss term.
Therefore, a judicious choice for $\lambda$ lies in the vicinity of $10^{-3}$.
Assigning an excessively large value for $\lambda$ may lead the model to disregard the classification loss.
\begin{figure}[h]
	\centering  
	\subfloat{
		\includegraphics[height=3cm, width=3.8cm]{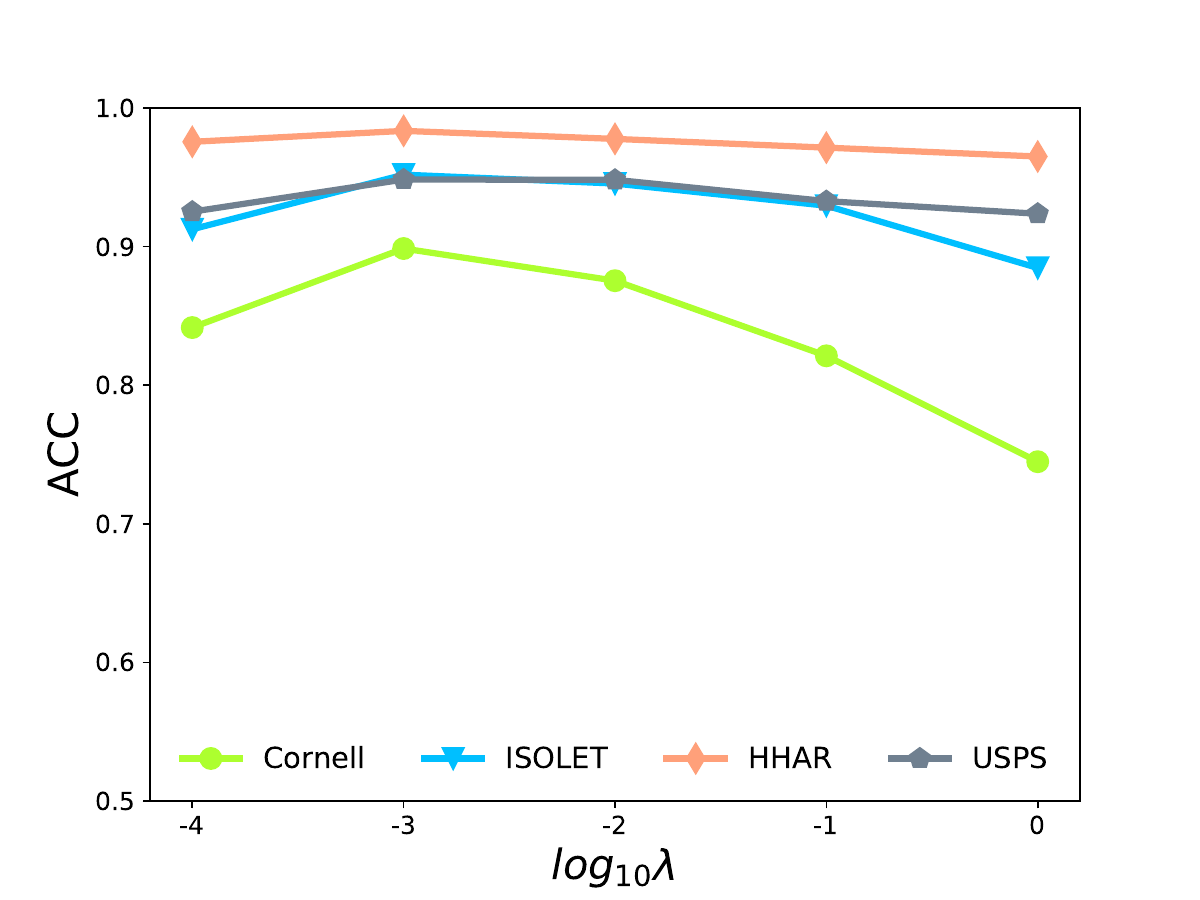}}
	\hspace{2mm} 
	\subfloat{
		\includegraphics[height=3cm, width=3.8cm]{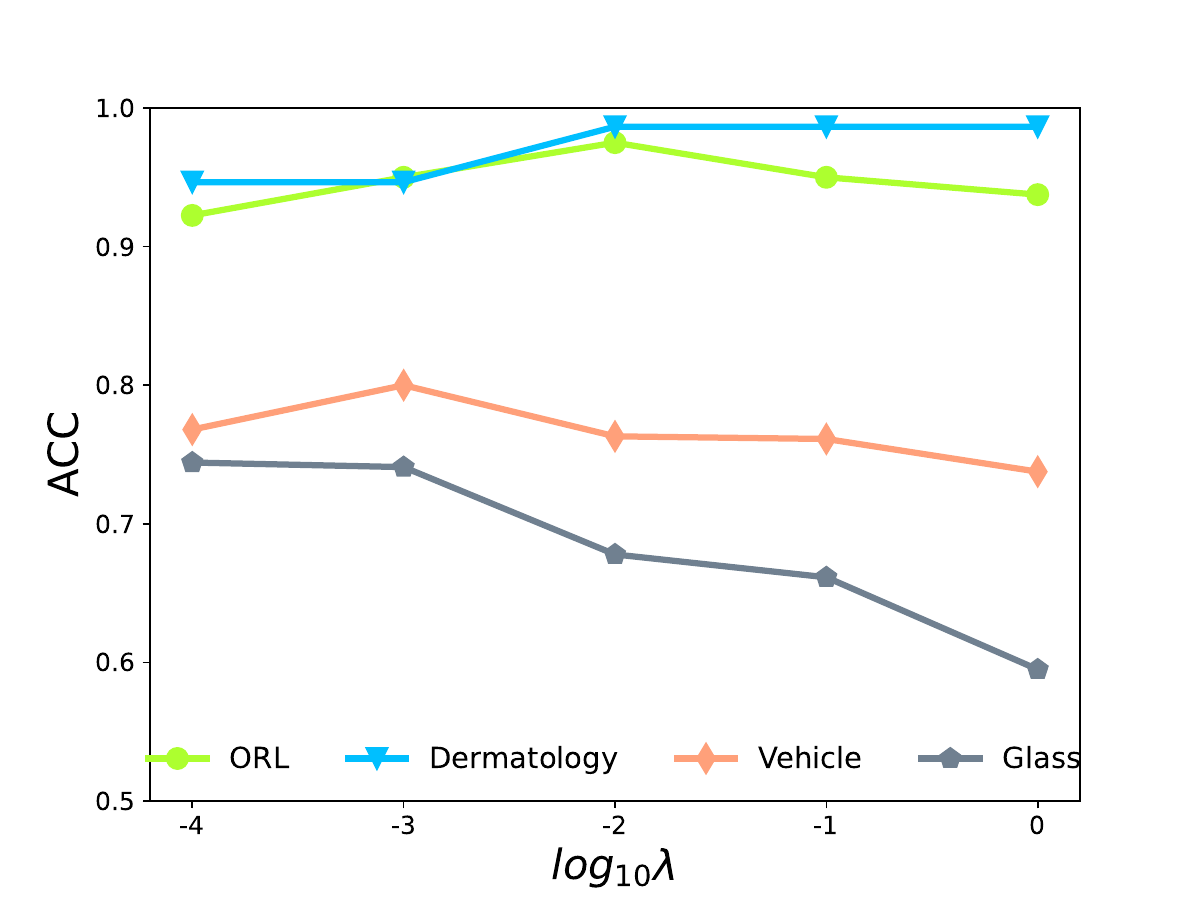}}
	\caption{Study of $\lambda$.} \label{ana_lambda}
	\vspace{-1mm}
\end{figure}
While the primary focus of this paper is on traditional methods, our method can be seamlessly integrated into the realm of deep learning.
We assess the enhancement of our method on softmax loss through visual classification tasks.
Illustrating representative outcomes, the displayed training and test accuracy curves in Figure. \ref{Accuracy curves with iterations on SVHN.} confirm the effectiveness of our method in mitigating overfitting.
Comprehensive descriptions, settings and results are presented in Appendix A.5.
Further discussions are placed in Appendix A.6.
\begin{figure}[h]
	\centering  
	\subfloat[Train]{
		\includegraphics[height=2.6cm, width=3.6cm]{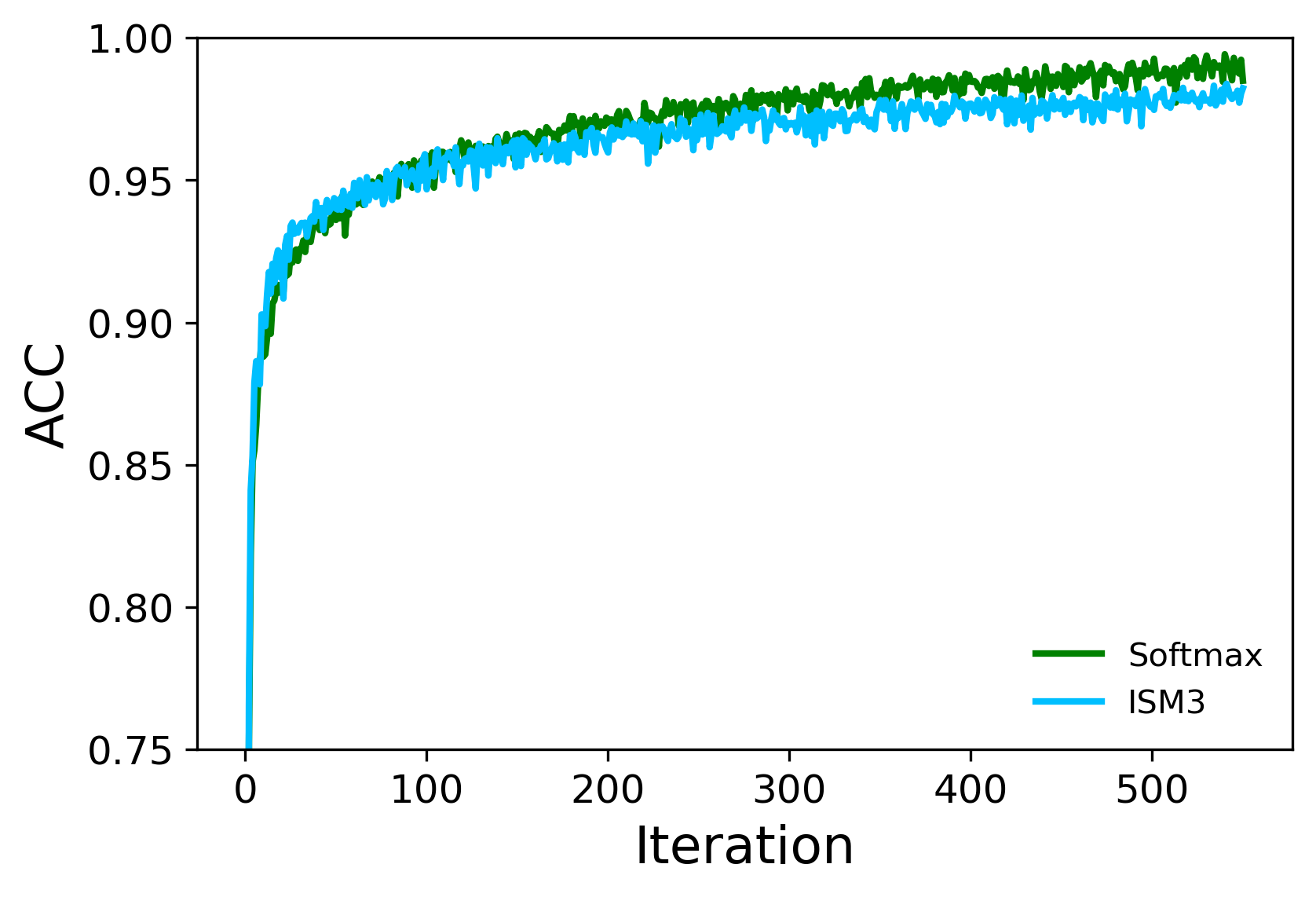}}
	\hspace{2mm} 
	\subfloat[Test]{
		\includegraphics[height=2.6cm, width=3.6cm]{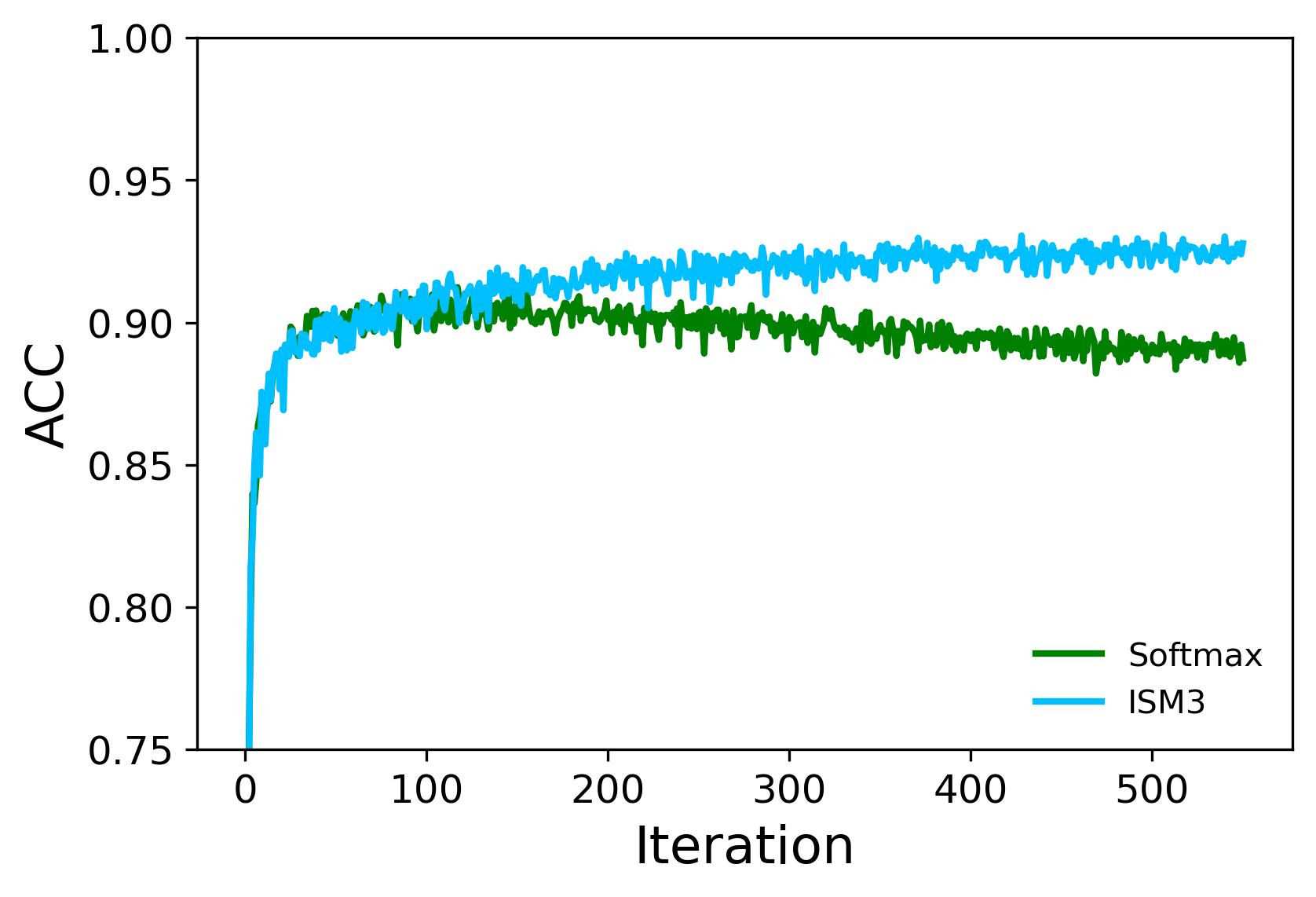}}
	\caption{Accuracy curves with iterations on SVHN.} \label{Accuracy curves with iterations on SVHN.}
	\vspace{-1mm}
\end{figure}

\section{Conclusion}
In this paper, we propose a concise but effective multi-class SVM model that enlarge the margin lower bound for all class pairs.
We reveal the drawbacks of the related methods, while providing the motivations and detailed derivations of our method.
Theoretical analysis confirms that the existing methods can be viewed as non-optimal special cases of our method.
We show the proposed method can broadly improve the generalization performance from the SRM perspective.
Our method can be integrated into neural networks, not only enhancing inter-class discrimination but also effectively mitigating overfitting. 
Both traditional and deep empirical evaluations validate the superiority of our method.

\bibliography{JabRef}

\appendix

\section{Appendix}
\setcounter{equation}{0}
\renewcommand\theequation{\arabic{equation}}

\setcounter{lemma}{0}
\renewcommand\thelemma{\arabic{lemma}}
\setcounter{theorem}{0}
\renewcommand\thetheorem{\arabic{theorem}}

\subsection{A.1 Proof}
\allowdisplaybreaks
%
%
%

\begin{lemma}\label{equ_p}
	Assume $ g_{1}(\mathbf{z}),g_{2}(\mathbf{z}),\cdots,g_{m}(\mathbf{z})$ are given real functions.
	The following two optimization problem is equivalent when $p\rightarrow \infty$:
	\begin{equation}
		(a) \max_{\mathbf{z}} \min_{j} g_{j}(\mathbf{z}) \qquad \qquad (b) \min_{\mathbf{z}} \sum_{j = 1}^{m} g^{-p}_{j}(\mathbf{z}).
	\end{equation}
\end{lemma}
\begin{proof}
	Let $g_{\alpha}(\mathbf{z})$ be temporary variable denoting the smallest elements in the set $ \{g_{j}(\mathbf{z}) \vert j \in [m]\} $.
	\begin{equation}\label{alpha_beta}
		\begin{aligned}
			&\min_{\mathbf{z}} \sum_{j = 1}^{m} g^{-p}_{j}(\mathbf{z}) \\
			\Longleftrightarrow
			&\min_{\mathbf{z}} \left(\sum_{j = 1}^{m} g^{-p}_{j}(\mathbf{z})\right)^{-1}\\
			\Longleftrightarrow
			&\min_{\mathbf{z}} g^{p}_{\alpha}(\mathbf{z}) \left(\sum_{j=1}^{m}  \left(\frac{g_{\alpha}(\mathbf{z})}{g_{j}(\mathbf{z})}\right)^{p}\right)^{-1}.
		\end{aligned}
	\end{equation}
	Because $p \rightarrow \infty$, for $j \neq \alpha$, $\left(\frac{g_{\alpha}(\mathbf{z})}{g_{j}(\mathbf{z})}\right)^{p} \rightarrow 0$.
	Then the problem (\ref{alpha_beta}) can be rewritten as
	\begin{equation}
		\min_{\mathbf{z}} \sum_{j = 1}^{m} g^{-p}_{j}(\mathbf{z}) \Longleftrightarrow \max_{\mathbf{z}} g^{p}_{\alpha}(\mathbf{z}) \Longleftrightarrow \max_{\mathbf{z}} \min_{j} g_{j}(\mathbf{z}).
	\end{equation}
	This concludes the proof.
\end{proof}

\begin{theorem}
	The problem 
	\begin{equation} \label{max_min}
		\begin{aligned}
			&\max_{\mathbf{W} \in \mathbb{R}^{d \times c},\mathbf{b} \in \mathbb{R}^{d}} \min_{k,l \in [c] , k < l}   \frac{1}{ \|\mathbf{w}_{k}-\mathbf{w}_{l}\|_{2}},\\
			{ s.t. }&\left\{ \begin{array}{cc}
				f_{kl}(\mathbf{x}_{i}) \geqslant 1, y_{i} = k\\ f_{kl}(\mathbf{x}_{i}) \leqslant -1, y_{i} = l
			\end{array}\right.,  i \in [n].
		\end{aligned}
	\end{equation}
	is equivalent to
	\begin{equation}\label{hard_p}
		\begin{aligned}
			&\min_{\mathbf{W} \in \mathbb{R}^{d \times c},\mathbf{b} \in \mathbb{R}^{d}} \sum_{k=1}^{c-1} \sum_{l=k+1}^{c} \|\mathbf{w}_{k} - \mathbf{w}_{l}\|_{2}^{p},\\
			&{ s.t. }\left\{ \begin{array}{cc}
				f_{kl}(\mathbf{x}_{i})  \geqslant 1, y_{i} = k\\ f_{kl}(\mathbf{x}_{i})  \leqslant -1, y_{i} = l
			\end{array}\right.,  i \in [n].
		\end{aligned}
	\end{equation}
	with the given parameter $p \rightarrow \infty$.
\end{theorem}
\begin{proof}
	According to Lemma \ref{equ_p}, let $f_{j}(\mathbf{W}) = \frac{1}{ \|\mathbf{w}_{k}-\mathbf{w}_{l}\|_{2}},k<l,j = 1,\cdots,c(c-1)/2$, the conclusion stands obviously.
\end{proof}

\begin{theorem}\label{th_constraints}
	Assume function $f(\mathbf{Z})$ has the property of column translation invariance, i.e., $\forall \bm{\sigma} \in \mathbb{R}^{n}$, there is $f(\mathbf{Z}) = f(\mathbf{Z} + \bm{\sigma} \mathbf{1}^{T}) $.
	With given $\varepsilon \rightarrow 0$, the following two optimization problems have the same optimal solution
	\begin{equation}\label{column}
		\min_{\mathbf{Z} \in \mathbb{R}^{n \times m}} f(\mathbf{Z}), \quad { s.t. } \sum_{j = 1}^{m} \mathbf{z}_{j} = \mathbf{0},
	\end{equation}
	\begin{equation}\label{penalty_Z}
		\min_{\mathbf{Z} \in \mathbb{R}^{n \times m}} f(\mathbf{Z}) + \varepsilon \|\mathbf{Z}\|_{F}^{2}.
	\end{equation}
\end{theorem}

\begin{proof}
	Before the proof, the following lemma is introduced.
	
	\textbf{Lemma.}
	With given vectors $\mathbf{z}_{1},\cdots,\mathbf{z}_{m} \in \mathbb{R}^{n}$, the optimal solution of problem
	\begin{equation}
		\min_{\bm{\sigma}\in \mathbb{R}^{n}}  \sum_{j = 1}^{m} \|\mathbf{z}_{j} - \bm{\sigma}\|^{2}
	\end{equation}
	is $\bm{\sigma} = \frac{1}{m}\sum_{j = 1}^{m} \mathbf{z}_{j}$.
	
	\begin{proof}
		Let $f(\bm{\sigma}) = \sum_{j = 1}^{m} \|\mathbf{z}_{j} - \bm{\sigma}\|^{2}$. Setting the derivative of $f(\bm{\sigma})$ with respect to $\bm{\sigma}$ to $0$, there is
		\begin{equation}
			\frac{\partial f}{\partial \bm{\sigma}} = -2 \sum_{j = 1}^{m} \mathbf{z}_{j} + 2 m\bm{\sigma} = \textbf{0}.
		\end{equation}
		Then the global minima is $\bm{\sigma} = \frac{1}{m}\sum_{j = 1}^{m} \mathbf{z}_{j}$.
	\end{proof}
	
	Consider the optimization problem
	\begin{equation}\label{min_Z}
		\min_{\mathbf{Z} \in \mathbb{R}^{n \times m}} g(\mathbf{Z}).
	\end{equation}
	Let $\mathbf{Z}^{*}$ be an arbitrary optimal solution of problem (\ref{min_Z}). Since $g$ has the property of column translation invariance,  $\forall \bm{\sigma} \in \mathbb{R}^{n}$, $\mathbf{Z}^{*} + \bm{\sigma} \mathbf{1}^{T}$ is also the optimal solution.
	Due to the constraints $\sum_{j = 1}^{m} \mathbf{z}_{j} = \mathbf{0}$, the only optimal solution of problem (\ref{column}) is
	\begin{equation}\label{opt_solution}
		\mathbf{Z}^{**} = \mathbf{Z}^{*} - \left(\frac{1}{m} \sum_{j = 1}^{m} \mathbf{z}^{*}_{j}\right) \mathbf{1}^{T}.
	\end{equation}
	According to the lemma above, $\mathbf{Z}^{**}$ is also the optimal solution of problem (\ref{penalty_Z}).
	Therefore, the two optimization problems have the same optimal solution.
\end{proof}

\begin{theorem}\label{th_Sigma}
	The following equation holds,
	\begin{equation}
		\sum_{k=1}^{c-1} \sum_{l=k+1}^{c} \sum_{y_{i}\in \{k,l\}} \left[1-y_{ikl} f_{kl}(\mathbf{x}_{i}) \right]_{+} = \sum_{i=1}^{n} \sum_{k \neq y_{i}}	[1 - f_{y_{i}k}(\mathbf{x}_{i}) ]_{+}.
	\end{equation}
\end{theorem}

\begin{proof}
	\begin{align}\label{Theorem 2}
		&\sum_{k=1}^{c-1} \sum_{l=k+1}^{c} \sum_{y_{i}\in \{k,l\}} \left[1-y_{ikl} f_{kl}(\mathbf{x}_{i}) \right]_{+} \notag \\
		=&\sum_{k=1}^{c-1} \sum_{l=k+1}^{c} \sum_{y_{i} = k}[1 - f_{kl}(\mathbf{x}_{i}) ]_{+}  + \sum_{k=1}^{c-1} \sum_{l=k+1}^{c} \sum_{y_{i} = l}[1 - f_{lk}(\mathbf{x}_{i}) ]_{+} \notag \\
		=&\sum_{k=1}^{c-1} \sum_{l=k+1}^{c} \sum_{y_{i} = k}[1 - f_{kl}(\mathbf{x}_{i}) ]_{+}  + \sum_{l=1}^{c-1} \sum_{k=l+1}^{c} \sum_{y_{i} = k}[1 - f_{kl}(\mathbf{x}_{i}) ]_{+} \notag \\
		=&\sum_{k=1}^{c-1} \sum_{l=k+1}^{c} \sum_{y_{i} = k}[1 - f_{kl}(\mathbf{x}_{i}) ]_{+}  + \sum_{k=2}^{c} \sum_{l=1}^{k-1} \sum_{y_{i} = k}[1 - f_{kl}(\mathbf{x}_{i}) ]_{+} \notag \\
		= & \sum_{k=1}^{c}  \sum_{l \neq k} \sum_{y_{i} = k}
		[1 - f_{kl}(\mathbf{x}_{i}) ]_{+} \notag \\
		= & \sum_{k=1}^{c} \sum_{y_{i} = k} \left(\sum_{l \neq k}
		[1 - f_{kl}(\mathbf{x}_{i}) ]_{+}\right) \notag \\
		= & \sum_{i=1}^{n} \sum_{l \neq y_{i}}	[1 - f_{y_{i}l}(\mathbf{x}_{i}) ]_{+},
	\end{align}
	which completes the proof.
\end{proof}

\begin{lemma}
	$g(x)$ satisfies $0 \leqslant g(x) - [x]_{+}  \leqslant \frac{\delta}{2}$.
	When $\delta \rightarrow 0$, $g(x) \rightarrow [x]_{+}$.
\end{lemma}
\begin{proof}
	Let $h(x) = g(x) - [x]_{+} $. For $x<0$, $h(x)$ monotonically increasing, so $0 <h(x)<h(0) = \frac{\delta}{2}$. For $x>0$, $h(x)$ monotonically decreasing, so $0 <h(x) <h(0) = \frac{\delta}{2}$.
\end{proof}

\begin{theorem}
	Problem (\ref{smooth_obj}) is strictly convex.
	\begin{equation}\label{smooth_obj}
		\begin{aligned}
			\min_{W \in \mathbb{R}^{d \times c}, \mathbf{b} \in \mathbb{R}^c} \mathcal{O}(\mathbf{W},\mathbf{b}) =  \sum_{i=1}^{n} \sum_{k \neq y_{i}}	\frac{\gamma_{ik} + \sqrt{\gamma_{ik}^2 + \delta^2}}{2}    \\
			+ \varepsilon \|\mathbf{W}\|_{F}^{2}  + \varepsilon \|\mathbf{b}\|_{2}^{2} + \lambda \sum_{k=1}^{c-1} \sum_{l=k+1}^{c} \|\mathbf{w}_{k} - \mathbf{w}_{l}\|_{2}^{p}.
		\end{aligned}
	\end{equation}
\end{theorem}
\begin{proof}
	The term $\varepsilon \|\mathbf{W}\|_{F}^{2}  + \varepsilon \|\mathbf{b}\|_{2}^{2}$ in Eq. (\ref{smooth_obj}) is obviously convex.
	It can be proved that functions $g(\mathbf{x}) = \|\mathbf{x}\|^{p}$ with $p \in (1,\infty)$ are totally convex in any uniformly convex Banach space.
	In finite dimensional spaces, totally convex deduces strictly convex.
	Since the Euclidean space $\mathbb{R}^{n}$  endowed with $p$-norm is uniformly convex with  $p \in (1,\infty)$, $h(\mathbf{x}) = \|\mathbf{x}\|_{2}^{p}$ is strongly convex in $\mathbb{R}^{n}$.
	Let $\mathcal{M}(\mathbf{W})$ be the last term in Eq. (\ref{smooth_obj}), for arbitrary $\mathbf{W},\mathbf{V} \in \mathbb{R}^{d \times c}$ and $\mu \in (0,1)$, we have
	\begin{equation}
		\begin{aligned}
			&\mathcal{M}(\mu\mathbf{W} + (1-\mu)\mathbf{V})\\
			= &\sum_{k=1}^{c-1} \sum_{l=k+1}^{c} h(\mu(\mathbf{w}_{k} - \mathbf{w}_{l}) + (1-\mu)(\mathbf{v}_{k} - \mathbf{v}_{l}))\\
			\leqslant &\sum_{k=1}^{c-1} \sum_{l=k+1}^{c} \mu h(\mathbf{w}_{k} - \mathbf{w}_{l}) + (1-\mu)h(\mathbf{v}_{k} - \mathbf{v}_{l})\\
			= & \mu \mathcal{M}(\mathbf{W}) + (1-\mu)\mathcal{M}(\mathbf{V})
		\end{aligned}
	\end{equation}
	Thereby $\mathcal{M}(\mathbf{W})$ is a strictly convex problem for $(\mathbf{W},\mathbf{b})$.
	Evidently, the first term in Eq. (\ref{smooth_obj}) is a convex approximation to hinge loss.
	Therefore, from the linear relationship $\gamma_{ik} = 1-(\mathbf{w}_{y_{i}}^{T}\mathbf{x}_i-\mathbf{w}_{k}^T \mathbf{x}_i+b_{y_{i}} -b_{k})$ between $\gamma_{ik}$ and $(\mathbf{W},\mathbf{b})$, it follows that the last term in Eq. (\ref{smooth_obj}) is strictly convex with respect to $(\mathbf{W},\mathbf{b})$.
\end{proof}

%

\begin{lemma}\label{variance}
	The following equation holds,
	\begin{equation}
		\sum_{j<k} \|\mathbf{w}_{j} - \mathbf{w}_{k}\|^{2} = c \sum_{j=1}^{c} \|\mathbf{w}_{j} - \frac{1}{c} \sum_{k=1}^{c}\mathbf{w}_{k}\|^{2}.
	\end{equation}
\end{lemma}
\begin{proof}
	\begin{align}
		&\sum_{j<k} \|\mathbf{w}_{j} - \mathbf{w}_{k}\|^{2} \notag \\
		= &  \frac{1}{2} \sum_{j=1}^{c} \sum_{k=1}^{c} \left(\mathbf{w}_{j}^{T} \mathbf{w}_{j}  -2 \mathbf{w}_{j}^{T} \mathbf{w}_{k} + \mathbf{w}_{k}^{T} \mathbf{w}_{k} \right) \notag \\
		= & c \sum_{j=1}^{c}  \left(\mathbf{w}_{j}^{T} \mathbf{w}_{j} - \frac{1}{c}  \mathbf{w}_{j}^{T} \sum_{k=1}^{c} \mathbf{w}_{k} \right) \notag \\
		= & c \left(\sum_{j=1}^{c}  \left(\mathbf{w}_{j}^{T} \mathbf{w}_{j} - \frac{2}{c} \sum_{k=1}^{c} \mathbf{w}_{j}^{T} \mathbf{w}_{k} \right) + \sum_{j=1}^{c} \frac{1}{c}  \mathbf{w}_{j}^{T} \sum_{k=1}^{c} \mathbf{w}_{k}\right) \notag \\
		= & c \sum_{j=1}^{c}  \left(\mathbf{w}_{j}^{T} \mathbf{w}_{j} - \frac{2}{c} \sum_{k=1}^{c} \mathbf{w}_{j}^{T} \mathbf{w}_{k} + \frac{1}{c^{2}} \sum_{j=1}^{c}  \mathbf{w}_{j}^{T} \sum_{k=1}^{c} \mathbf{w}_{k} \right) \notag  \\
		= & c \sum_{j=1}^{c} \|\mathbf{w}_{j} - \frac{1}{c} \sum_{k=1}^{c}\mathbf{w}_{k}\|^{2},
	\end{align}
	which completes the proof.
\end{proof}

\begin{theorem}
	The optimization problem (\ref{i_obj}) with $p=2$ has the same optimal solution as
	\begin{equation} \label{sq_obj}
		\min_{\mathbf{W}, \mathbf{b} } \sum_{i=1}^{n} \sum_{k \neq y_{i}}	[1 - f_{y_{i}k}(\mathbf{x}_{i}) ]_{+} +   \lambda c \sum_{k=1}^c\left\|\mathbf{w}_{k}\right\|_2^2  + \varepsilon \|\mathbf{b}\|_{2}^{2}.
	\end{equation}
	
\end{theorem}
\begin{proof}
	According to Theorem \ref{th_Sigma}, problem
	\begin{equation} \label{obj}
		\begin{aligned}
			\min_{\mathbf{W}  \in \mathbb{R}^{d \times c}, \mathbf{b} \in \mathbb{R}^c} \sum_{k<l} \sum_{y_{i}\in \{k,l\}} [1-y_{ikl}f_{kl}(\mathbf{x}_{i}) ]_{+} \\
			+ \lambda \sum_{k<l} \|\mathbf{w}_{k}-\mathbf{w}_{l}\|_{2}^{p} + \varepsilon \left(\|\mathbf{W}\|_{F}^{2} + \|\mathbf{b}\|_{2}^{2}\right).
		\end{aligned}
	\end{equation}
	is equivalent to
	\begin{equation} \label{i_obj}
		\begin{aligned}
			\min_{\mathbf{W}  \in \mathbb{R}^{d \times c}, \mathbf{b} \in \mathbb{R}^c} \sum_{i=1}^{n} \sum_{k \neq y_{i}}	[1 - f_{y_{i}k}(\mathbf{x}_{i}) ]_{+} + \\
			 \lambda\sum_{k<l} \|\mathbf{w}_{k}-\mathbf{w}_{l}\|_{2}^{p} + \varepsilon \left(\|\mathbf{W}\|_{F}^{2} + \|\mathbf{b}\|_{2}^{2}\right).
		\end{aligned}
	\end{equation}
	According to Theorem \ref{th_constraints}, when $\mathbf{W}$ in problem (\ref{i_obj}) is taken to be optimal, there is $\sum_{k = 1}^{c}\mathbf{w}_{k} = \mathbf{0}$.	
	Due to the conclusion of Lemma \ref{variance}, with $p = 2$, the optimization problem (\ref{i_obj}) is equivalent to
	\begin{equation}\label{c_eps}
		\begin{aligned}
			\min_{\mathbf{W} \in \mathbb{R}^{d \times c}, \mathbf{b} \in \mathbb{R}^c} \sum_{i=1}^{n} \sum_{k \neq y_{i}}	[1 - f_{y_{i}k}(\mathbf{x}_{i}) ]_{+}\\
			 + \lambda(c + \varepsilon)  \sum_{k=1}^c\left\|\mathbf{w}_{k}\right\|_2^2     + \varepsilon \|\mathbf{b}\|_{2}^{2}.
		\end{aligned}
	\end{equation}
	Because $\varepsilon \rightarrow 0$, problem (\ref{sq_obj}) and problem (\ref{c_eps}) have the same optimal solution.
	This concludes the proof.

	%
\end{proof}

\begin{theorem}
	Let $\mathcal{F}$ be the multivariate linear model from $\mathcal{X}$ into $\mathbb{R}^{c}$.
	$\mathcal{F}$ are endowed  with the Euclidean norm.
	If $\mathcal{X}$ is included in a ball of radius $\Lambda_{\mathcal{X}}$ about the origin, then $\forall f \in \mathcal{F}$(parametrized by $\mathbf{W}$ and $\mathbf{b}$) the following bound holds:
	\begin{equation}
		\|T(f)\|_{l_{\infty}, l_{p}} \leqslant  \Lambda_{\mathcal{X}} \left(\sum_{k<l}\|\mathbf{w}_{k} - \mathbf{w}_{l}\|_{2}^{p}\right)^{\frac{1}{p}}.
	\end{equation}
\end{theorem}

\begin{proof}
	According to the definition of functional $T$, $\forall \mathbf{x} \in \mathcal{X}$, $\forall f \in \mathcal{F}$,
	\begin{equation}
		\begin{aligned}
			\|T(f)\|_{l_{\infty}, l_{p}} &= \max_{i \in [2n]} \left(\sum_{k<l} | (\mathbf{w}_{k} - \mathbf{w}_{l})^{T} \mathbf{x}_{i}|^{p}\right)^{\frac{1}{p}}\\
			&\leqslant \max_{i \in [2n]} \left(\sum_{k<l} \| \mathbf{w}_{k} - \mathbf{w}_{l} \|_{2}^{p} \|\mathbf{x}_{i}\|_{2}^{p}\right)^{\frac{1}{p}}\\
			&\leqslant \Lambda_{\mathcal{X}} \left(\sum_{k<l}\|\mathbf{w}_{k} - \mathbf{w}_{l}\|_{2}^{p}\right)^{\frac{1}{p}},
		\end{aligned}
	\end{equation}
	which completes the proof.
\end{proof}

\subsection{A.2 Illustrative experiments}
To demonstrate the varying levels of difficulty in distinguishing between different classes, we conducted a few supplementary experiments.
Figure \ref{supplementary_tsne} displays the t-SNE visualization results of several commonly used datasets.
It can be observed that the reduced-dimensional representations with greater confusion are concentrated in certain localized areas, such as those enclosed by the black boxes.
\begin{figure} [H]
	\centering
	\subfloat[COIL20]{
		\includegraphics[height=2.2cm, width=2.4cm]{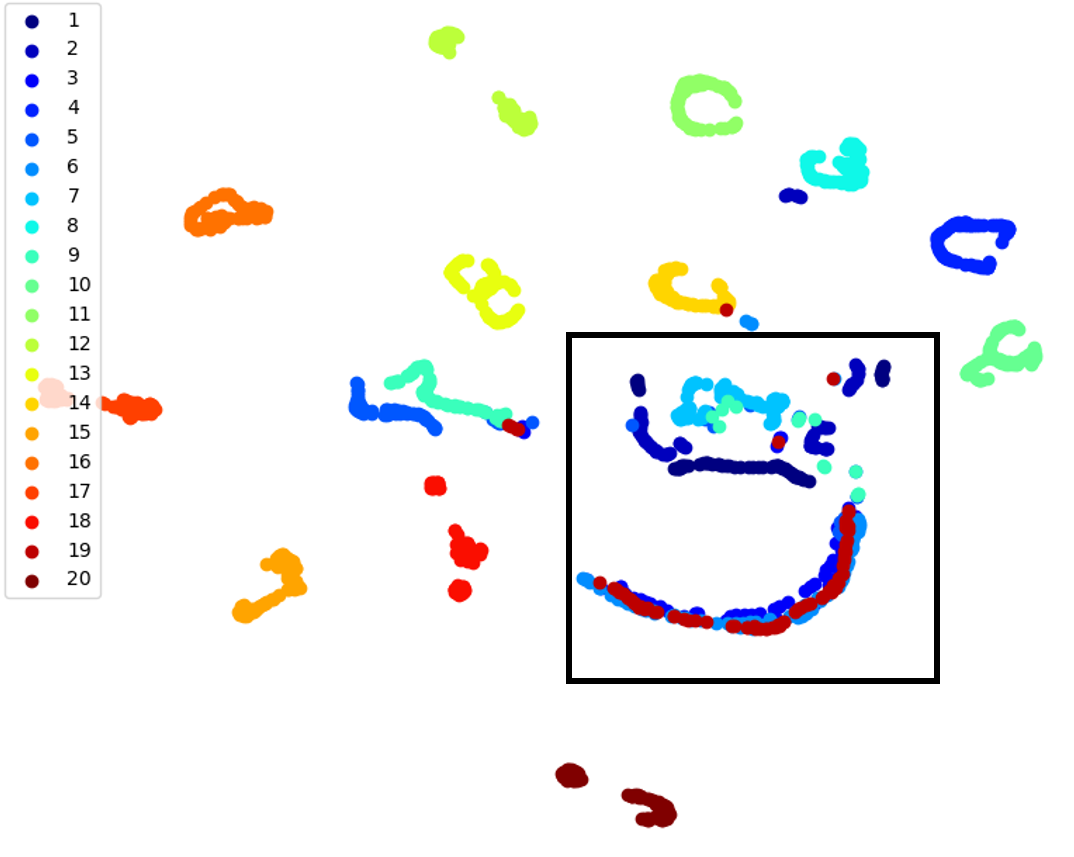} \label{tsne_COIL20}}  
	\hspace{2mm} 
	\subfloat[Dermatology]{
		\includegraphics[ height=2cm, width=2.4cm]{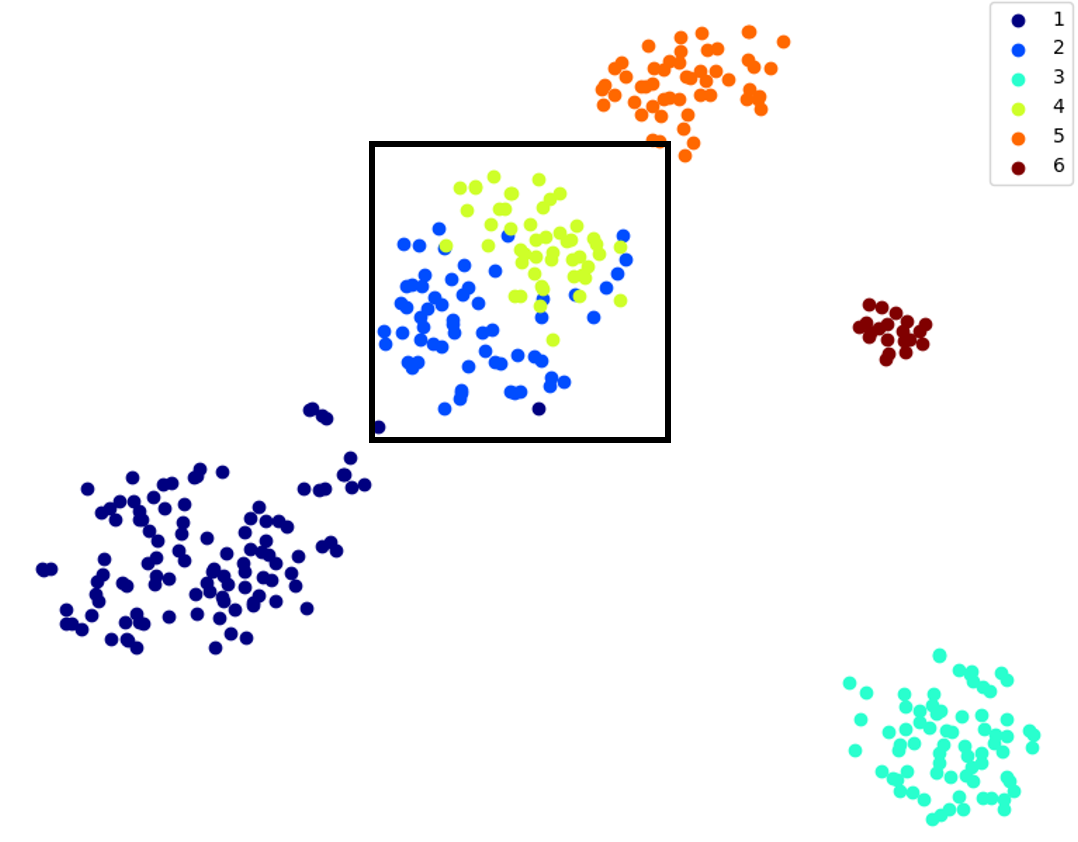} 
		\label{tsne_Der}}
	\hspace{2mm} 
	\subfloat[Segment.]{
		\includegraphics[height=2.2cm, width=2.4cm]{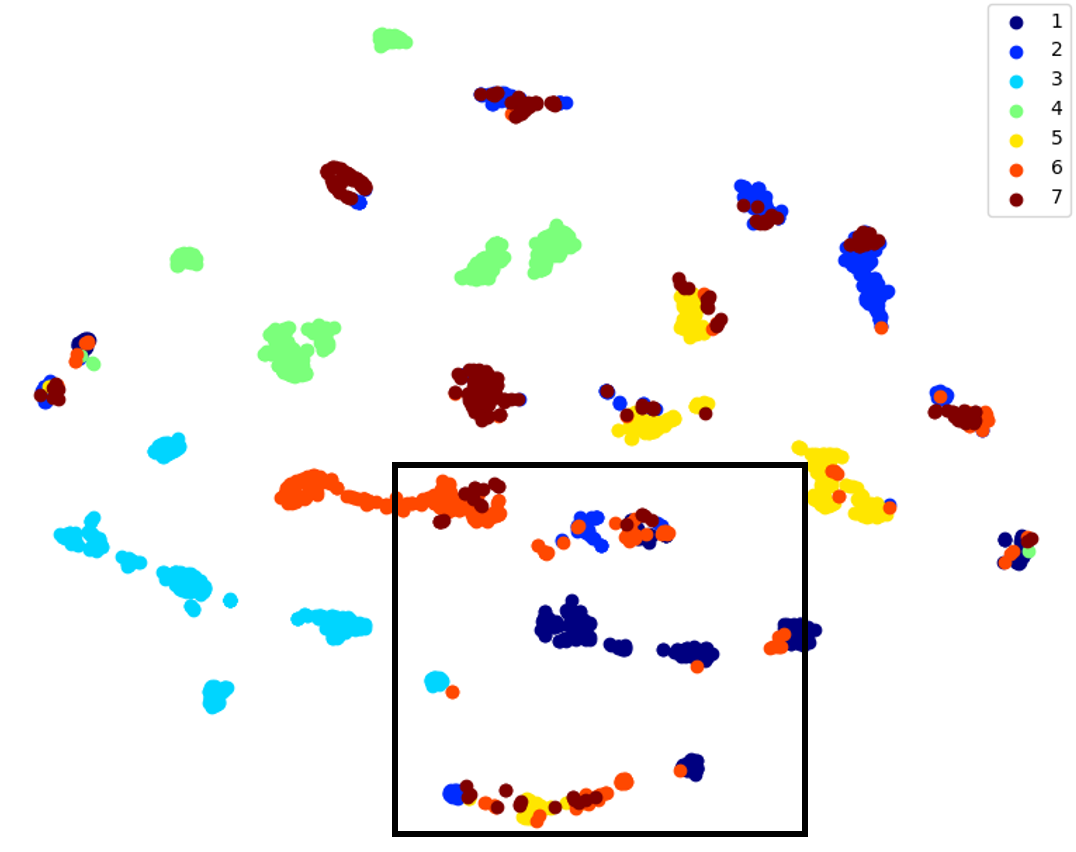} \label{tsne_segment}}
	\caption{Two-dimensional visualization of common datasets.}
	\label{supplementary_tsne}
\end{figure}

On commonly used datasets, using the OvO strategy for binary logistic regression, the confusion matrices of several results are shown in Figure \ref{Confusion}.
It can be observed that the non-zero values on the off-diagonal are concentrated in certain subsets (the displayed confusion matrices have undergone a rearrangement of classes), indicating that misclassification is mainly concentrated in some subsets (within rad boxes) of all classes.
\begin{figure} [H]
	\centering
	\subfloat[Iris with OvO LR.]{
		\includegraphics[height=2.2cm, width=2.4cm]{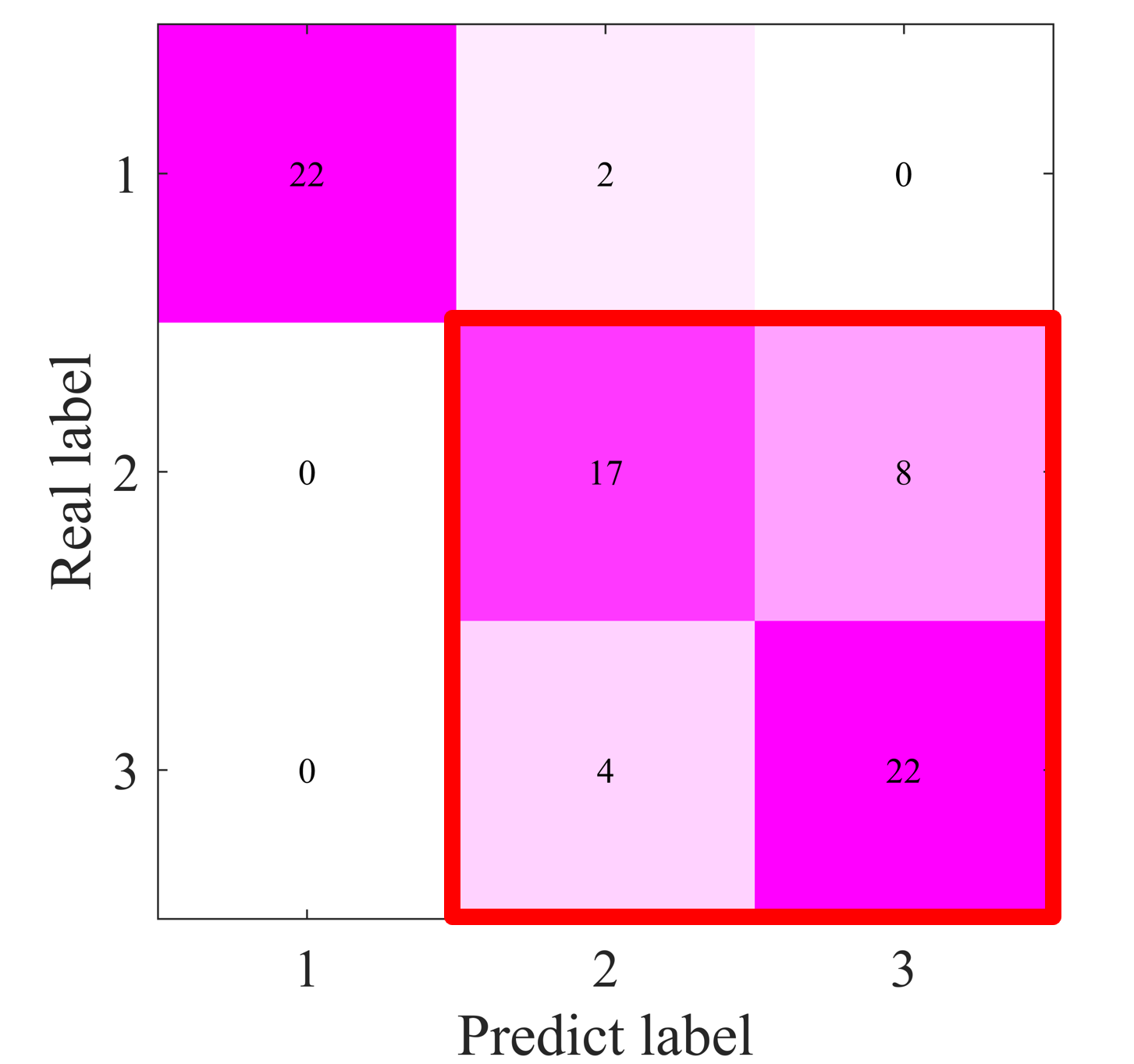} \label{confusion_SVM_Iris}}  
	\hspace{2mm} 
	\subfloat[Dermatology with OvO LR.]{
		\includegraphics[ height=2.2cm, width=2.4cm]{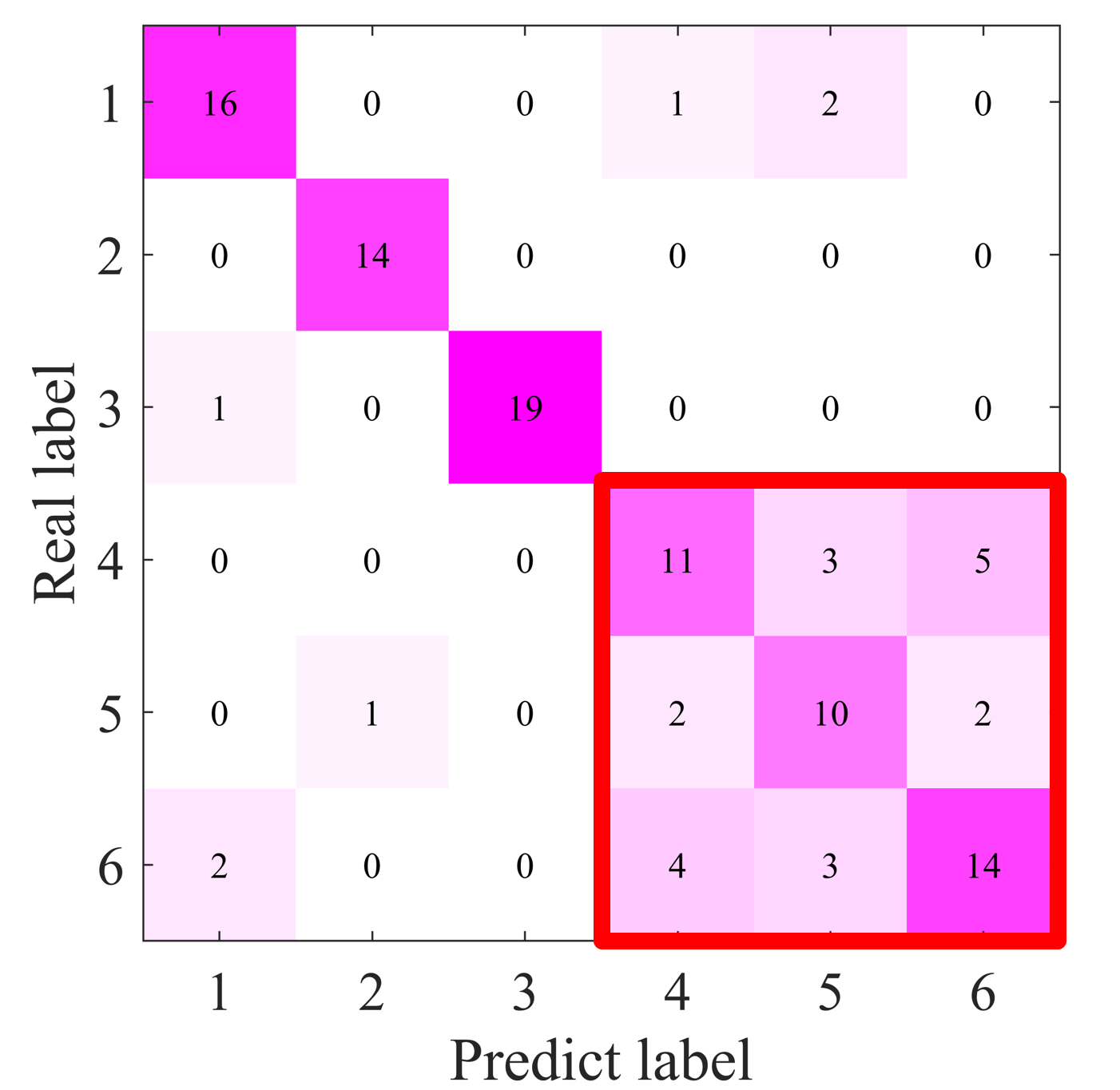} 
		\label{confusion_SVM_Dermatology}}
	\hspace{2mm} 
	\subfloat[Optdigits with OvO LR.]{
		\includegraphics[height=2.2cm, width=2.4cm]{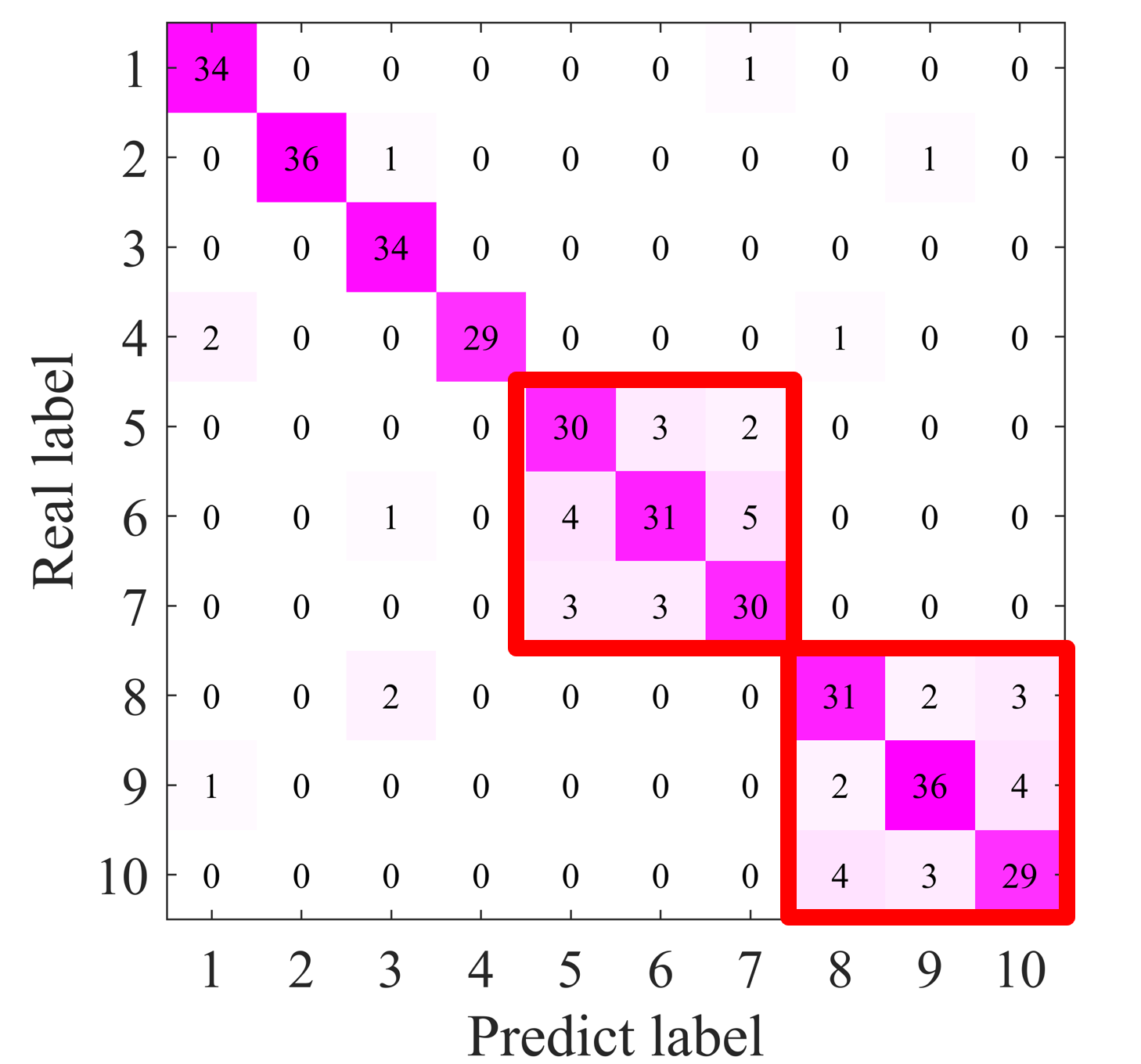} \label{confusion_SVM_Optdigits}}
	\caption{Confusion matrices by OvO strategy.}
	\label{Confusion}
\end{figure}
To pay more attention on those "pivotal" class pairs, we adopt maximizing margin lower bound.

\subsection{A.3 Extension to softmax loss}

Logistic regression transforms the input features into a probability-valued output for each class:	
\begin{equation}
	p(y \mid \mathbf{x})=\frac{1}{1+\exp \left(-y\left(\mathbf{w}^T \mathbf{x}+b\right)\right)},\quad y \in \{-1, 1\}.
\end{equation}
The model parameters can be estimated by the maximum likelihood estimation method $\max_{\mathbf{w}, b} \prod_{i=1}^{n}  p(y_{i} \mid \mathbf{x}_{i})$.
Converting it into an easily solvable logarithmic form, the standard logistic regression turns out to be the following problem:
\begin{equation}
	\min_{\mathbf{w}, b} \mathcal{L}(\mathbf{w}, b) = \sum_{i=1}^{n}  \log\left(1+\exp \left(-y_{i}\left(\mathbf{w}^T \mathbf{x}_{i}+b\right)\right)\right)
\end{equation}
Let $f(t) = log(1+\exp(-t))$, then $f(y_{i}(\mathbf{w}^T \mathbf{x}_{i}+b))$ represents the loss of sample $\mathbf{x}_{i}$.
The $\ell_{2}$-regularized logistic regression is trained form the optimization problem
\begin{equation}\label{RLR_loss}
	\min_{\mathbf{w}, b} \sum_{i =1}^{n} f(y_{i}(\mathbf{w}^T \mathbf{x}_{i}+b)) + \lambda \|\mathbf{w}\|^{2}.
\end{equation}
We now explicate that minimizing the regularizer $\|\mathbf{w}\|^{2}$ not only regulates the model's complexity  but also resembles the maximization of the "margin", the concept adopted in SVM.
%
%
%
Firsty, the linear decision hyperplane for binary logistic regression is $\mathbf{w}^T \mathbf{x}+b = 0$.
Specifically, for $y_{i} = 1$, $p(y_{i} \mid \mathbf{x}_{i}) > 0.5 \Leftrightarrow \mathbf{w}^T \mathbf{x}_{i}+b > 0$ and vice versa.
The converse holds for $y_{i} = -1$.
Assuming the logistic regression loss function to be lower than a stable positive real number $\varepsilon$, and with the hyperplane normal vector (i.e. $\mathbf{w}$) fixed, the solution set of the bias term $b$ can be expressed as $\mathcal{B} = \{b \in \mathbb{R} | \mathcal{L}(\mathbf{w}, b) \leqslant \varepsilon\}$.
\begin{lemma}
	By definition, if non-empty, set $\mathcal{B}$ is connected and has a maximum and a minimum.	
\end{lemma}
\begin{proof}
	To show that the set $\mathcal{B}$ is connected, we begin by examining the properties of the function $\mathcal{L}$.
	
	\begin{equation}
		\mathcal{L}(\mathbf{w}, b) = \sum_{i=1}^{n}  \log\left(1+\exp \left(-y_{i}\left(\mathbf{w}^T \mathbf{x}_{i}+b\right)\right)\right).
	\end{equation}
	With samples and normal vector of the hyperplane (i.e. $\mathbf{w}$) fixed, it can be simplified into
	\begin{equation}
		\mathcal{L}(b) = n_{1} \log\left(1+\exp \left(-b\right)\right) + (n-n_{1}) \log\left(1+\exp \left(b\right)\right),
	\end{equation}
	where $n_{1}$ represents the number of positive instances in the training set.
	By taking the derivative, it can be deduced that $\mathcal{L}(b)$ first decreases and then increases in $\mathcal{R}$.
	Therefore, by the Intermediate Value Theorem, it is evident that $\mathcal{B}$ is connected.
	
	$\mathcal{B}$ is a closed set by definition.
	Besides, there is $\lim_{b \rightarrow \infty}\mathcal{L}(b) = \infty$, so $\mathcal{B}$ is bounded.
	According to the Extreme Value Theorem, $\mathcal{B}$ has a maximum and a minimum.
\end{proof}

\begin{figure} [H]
	\centering
	\subfloat[Margin in SVM.]{
		\includegraphics[height=2.4cm, width=3.2cm]{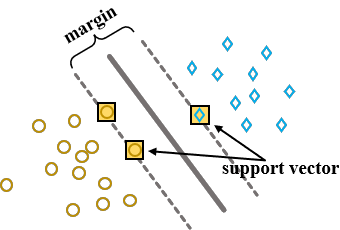} \label{margin_SVM}}  
	\hspace{5mm} 
	\subfloat[Margin in LR.]{
		\includegraphics[ height=2.4cm, width=3.2cm]{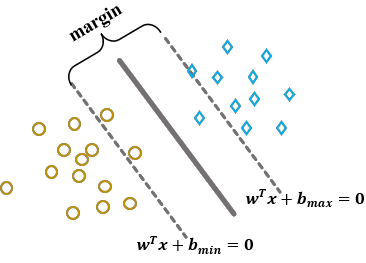} \label{margin_LR}}
	\\
	\vspace{4mm} 
	\subfloat[Multi-class case.]{
		\includegraphics[ height=4cm, width=6cm]{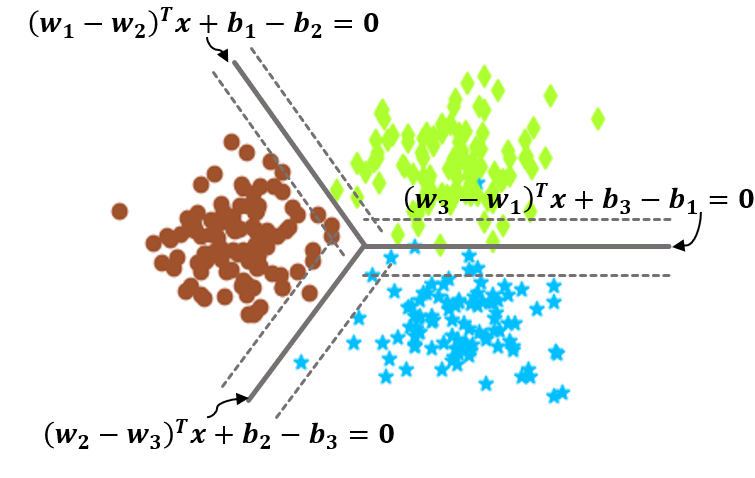} \label{multi_class}}
	\caption{Some illustrative figures.}
\end{figure}

Let $b_{max}$ and $b_{min}$ denote the maximum and minimum in $\mathcal{B}$, respectively.
In SVM, margin is intuitively defined as the shortest distance between two convex hulls of the two classes.
Analogously, as illustrated in Figure \ref{margin_SVM} \ref{margin_LR}, the margin in logistic regression can be defined as the distance between the two farthest hyperplanes in $\mathcal{B}$:
\begin{equation}
	margin = \frac{b_{max} - b_{min}}{\|\mathbf{w}\|}
\end{equation}
Multiplying $\mathbf{w}$ by a constant does not affect the actual hyperplane, so let $\mathbf{w} = \frac{\mathbf{w}}{b_{max} - b_{min}}$.
Therefore the margin maximization problem can be written as
\begin{equation}\label{loss_leq}
	\max_{\mathbf{w}, b} \frac{1}{\|\mathbf{w}\|} \quad \text{ s.t. } \mathcal{L}(\mathbf{w}, b) \leqslant \varepsilon, \varepsilon > 0.
\end{equation}
This is similar to the soft-margin SVM.
Nevertheless, the upper bound of the loss $\varepsilon$ is not possible to be pre-determined and should be optimized.
Therefore, $\varepsilon$ can be incorporated into the optimization problem:
\begin{equation}\label{neq_opt}
	\min_{\mathbf{w}, b, \varepsilon} \|\mathbf{w}\|^{2} + C \varepsilon \quad \text{ s.t. } \mathcal{L}(\mathbf{w}, b) \leqslant \varepsilon, \varepsilon > 0,
\end{equation}
where $C$ is the trade-off parameter between minimizing loss and maximizing margin.
Obviously, by letting $\varepsilon = \mathcal{L}(\mathbf{w}, b)$ and $\lambda = \frac{1}{C}$, problem (\ref{RLR_loss}) is equivalent to problem (\ref{neq_opt}).
This implies that in the case where the training loss can no longer be significantly reduced, problem (\ref{RLR_loss}) tends to search for a hyperplane with large margin, with the importance of both balanced by $\lambda$.

To summarize the discussion, if the hinge loss in M$^3$SVM is replaced by a logistic regression loss, the complete model can be written as
\begin{equation}\label{overall_obj}
	\begin{aligned}
		\min_{\mathbf{W} \in \mathbb{R}^{d \times c}, \mathbf{b} \in \mathbb{R}^{c}} \sum_{j<k}   \sum_{y_{i} \in \{j,k\}} f\left[y_{i}\left((\mathbf{w}_{j} - \mathbf{w}_{k})^T \mathbf{x}_{i}+b_{j} - b_{k}\right)\right]\\
		 + \lambda  \sum_{j<k}  \|\mathbf{w}_{j} - \mathbf{w}_{k}\|^{p}  + \delta \left(\|\mathbf{W}\|_{F}^{2} + \|\mathbf{b}\|^{2}\right)  
	\end{aligned}
\end{equation}

\noindent \textbf{Connection to softmax loss:}
\begin{theorem}\label{th_Sigma}
	The following equation holds,
	\begin{equation}\label{trans}
		\begin{aligned}
			\sum_{j<k}   \sum_{y_{i} \in \{j,k\}} f\left[y_{i}\left((\mathbf{w}_{j} - \mathbf{w}_{k})^T \mathbf{x}_{i}+b_{j} - b_{k}\right)\right] \\
			= \sum_{i=1}^{n} \sum_{j \neq y_{i}}	f \left((\mathbf{w}_{y_{i}} - \mathbf{w}_{j})^T \mathbf{x}_{i}+b_{y_{i}} - b_{j}\right).
		\end{aligned}
	\end{equation}
\end{theorem}

Following most existing works, we equally define the softmax loss as the cross-entropy loss of the last fully connected layer after the softmax function with the true label, shown in Figure \ref{arch}.
\begin{figure*}[h]\label{arch}
	\centering
	\includegraphics[height = 2cm, width=16cm]{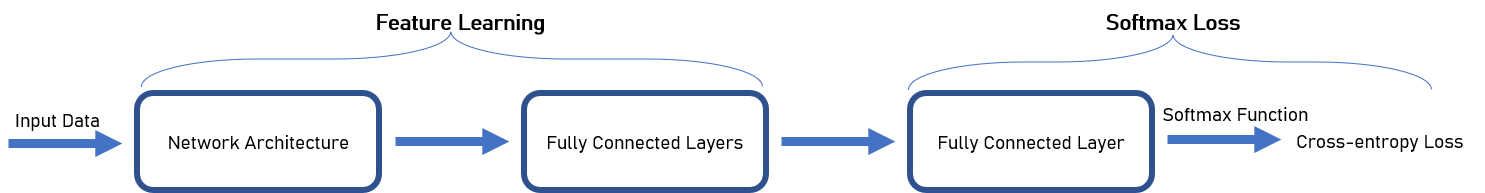}
	\caption{The architecture for the proposed ISM3.}
	\label{arch}
\end{figure*}

\begin{lemma}
	Softmax loss can be written as $\sum_{i=1}^{n} \log \left( 1 + \sum_{j \neq y_{i}} \exp((\mathbf{w}_{j} - \mathbf{w}_{y_{i}})^{T} \mathbf{x}_{i} + b_{j} - b_{y_{i}}) \right)$.
\end{lemma}
We consider the relationship between the loss term in Eq. (\ref{trans}) and softmax loss.
It can be found that the two loss functions differ only in the position of the summation sign.
If the log function is approximated by its first order Taylor expansion, the two loss functions then become identical, which is $\sum_{i=1}^{n} \sum_{j \neq y_{i}} \exp((\mathbf{w}_{j} - \mathbf{w}_{y_{i}})^{T} \mathbf{x}_{i} + b_{j} - b_{y_{i}}) $.
In other words, the two distinct models achieve the purpose of class discrimination with isomorphic loss function.
On this basis, our method explicitly enlarges the inter-class margins, hence it is referred to as \textbf{I}mproved \textbf{S}oftmax loss via \textbf{M}aximizing \textbf{M}inimum \textbf{M}argin (ISM3).

\subsection{A.4 Datasets} 
The details of the datasets used in our experiments are described in Table \ref{datasets}.
\begin{table}[H]
	\begin{center}
		\caption{Details of the experimental datasets. } 
		\label{datasets}
		\resizebox{\linewidth}{26mm}{
		\begin{tabular}{ccccc}
			\toprule[1pt]
			\textbf{Dataset}   & Instances & Features & Classes & Data Type  \\ 
			\hline
			\slshape{Cornell} & 827      & 4143     & 7  & Document \\
			\slshape{ISOLET} & 1560      & 26     & 617  & Speech \\
			\slshape{HHAR} & 10229     & 561    & 6  & Movement Signal \\
			\slshape{USPS} & 1854      & 256      & 10  & Handwritten Image \\
			\slshape{ORL} & 400      & 1024      & 40  & Face Image \\
			\slshape{Dermatology}        & 366      & 34       & 6    & Attribute   \\
			\slshape{Vehicle}    & 846      & 18      & 4  & Attribute    \\
			\slshape{Glass}       & 214       & 9     & 6  & Attribute    \\
			\hline
		\end{tabular}
	}
	\end{center}
\end{table}

\subsection{A.5 Integrated into deep learning}
Softmax loss has been widely utilized in various deep learning tasks such as visual classification, language modeling, and generative modeling.
In this paper, we assess the enhancement of our method on softmax loss through visual classification task.
To be fair, we chose ISM3 (logistic regression version) over M$^3$SVM for the comparison.
ISM3 loss in the final layer also facilitates the learning of a large inter-class margin embedding $\mathbf{x}$ through the backpropagation.
To reiterate, neural networks are not utilized as feature extractors for ISM3 loss; rather, ISM3 loss directs the learning of network parameters.
Hence, the stacked CNN layers are not a frozen backbone, but rather updatable.

\begin{table}[h]
	\centering
	\caption{Simple CNN architecture for different benchmark datasets.}
	\label{table:network_architecture}
	\resizebox{\linewidth}{9mm}{
		\begin{tabular}{|c|c|c|c|c|c|c|}
			\hline
			Layers & Filters & Kernel\_size & Stride & Padding & Max\_Pooling & Stacking                  \\ \hline
			layer1 & 16      & 5            & 1      & 2       & kernel\_size = 2, stride = 2 & 1\\ \hline
			layer2 & 32      & 5            & 1      & 2       & kernel\_size = 2, stride = 2 & 1\\ \hline
			layer3 & 64      & 5            & 1      & 2       & kernel\_size = 2, stride = 2 & 3\\ \hline
		\end{tabular}
	}
\end{table}

For the sake of simplicity and consistency, the parameter $p$ of ISM3 is uniformly set to $4$ without further adjustments.
The visual classification task encompasses several datasets of comparable size, including FashionMNIST, PlantVillage, SVHN, STL10, CIFAR10.
They are all locatable within Deeplake \footnote{\url{https://github.com/activeloopai/deeplake}}.
To illustrate the effectiveness, we employ the convolutional neural network comprising merely five layers as the feature extractor and substitute the softmax loss with ISM3 for comparison experiments.
The specific network architecture is presented in Table \ref{table:network_architecture}.
The results on the lightweight visual classification task are presented in Table \ref{table:visual_classification}.

\begin{table}[h]\large
	\begin{center}
		\caption{Average performance (w.r.t.  Accuracy) on test set for visual classification.}
		\label{table:visual_classification}
		\resizebox{\linewidth}{9.2mm}{
			\begin{tabular}{lccccc}
				\toprule[2pt]
				Datasets         
				& FashionMNIST & PlantVillage  & SVHN  &STL10   & CIFAR10           \\ \hline
				\slshape{Softmax}     	&0.8856 &0.9006	&  0.8573  & 0.5461    & 0.6808  \\ \hline
				\slshape{ISM3}	 & 0.9221 &0.9275  &  0.9184  & 0.5814  & 0.7163    \\
				\hline
			\end{tabular}
		}
	\end{center}
\end{table}

\begin{figure}[H]
	\setlength{\abovecaptionskip}{2mm}
	\setlength{\belowcaptionskip}{0mm}
	\centering
	\subfloat[FashionMNIST]{
		\includegraphics[height=2.2cm, width=2.7cm]{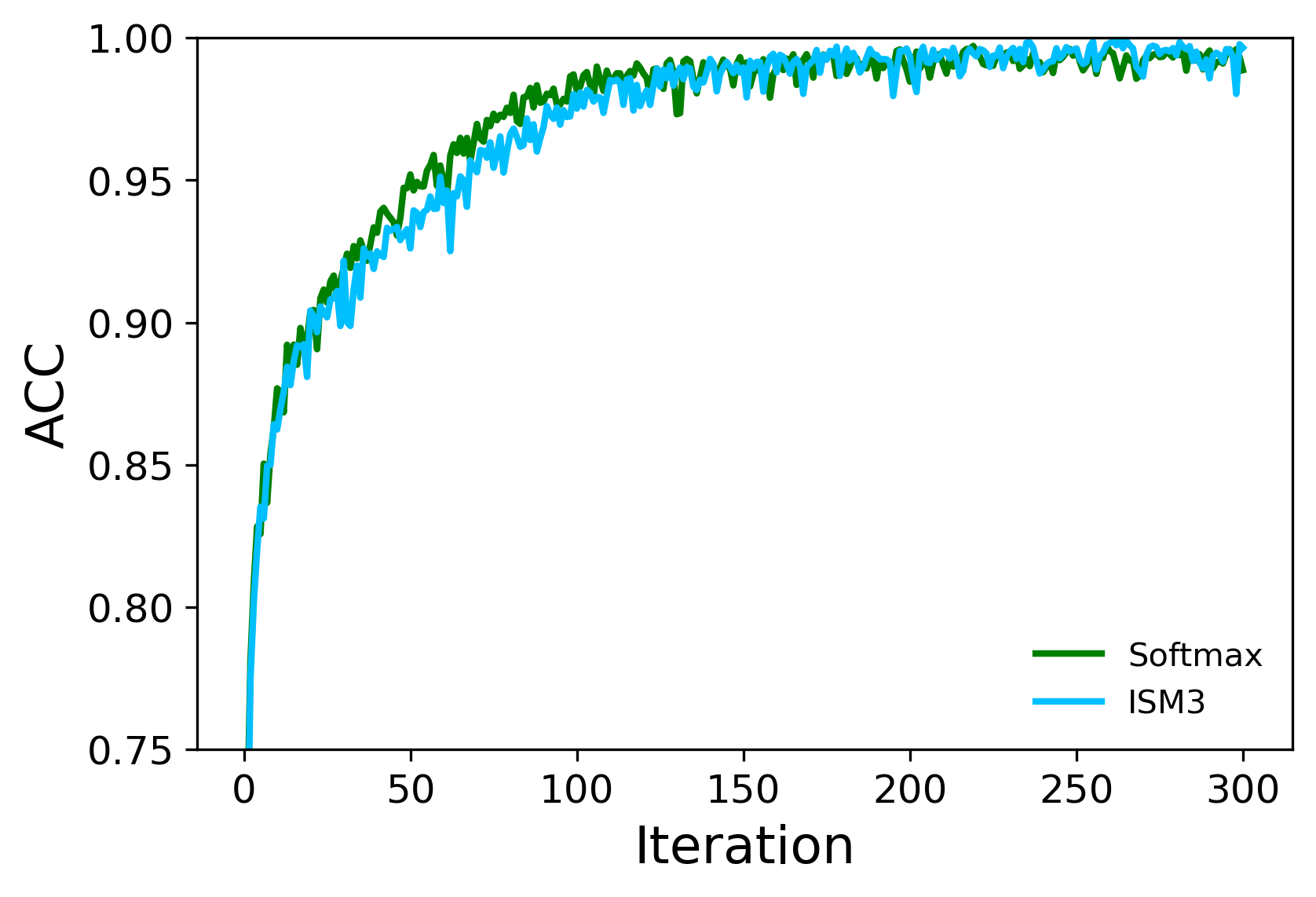} \label{initial}}
	\subfloat[SVHN]{
		\includegraphics[height=2.2cm, width=2.7cm]{SVHN_train_acc.png} \label{with high uniformity}}
	\subfloat[CIFAR10]{
		\includegraphics[height=2.2cm, width=2.7cm]{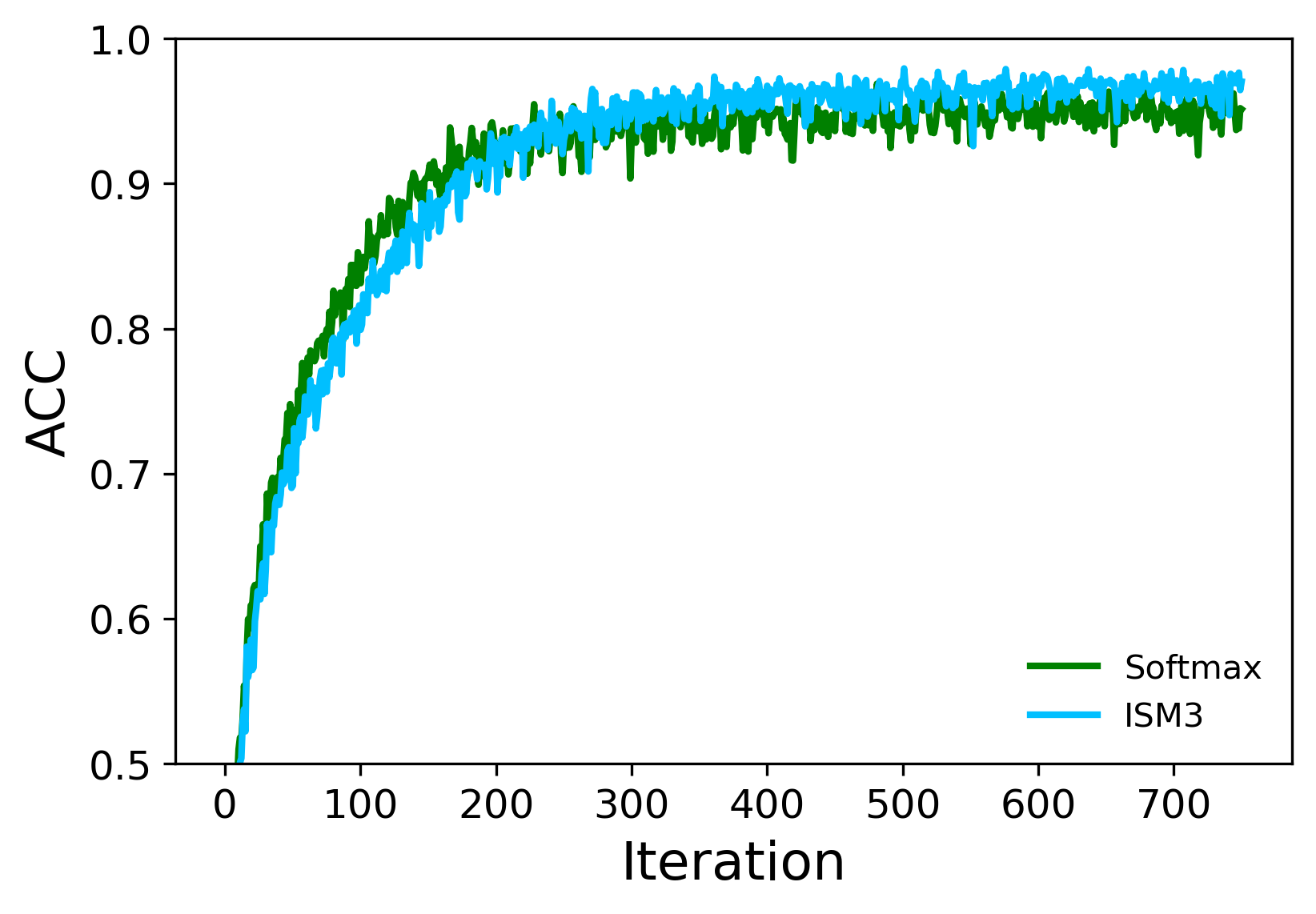} \label{with high uniformity}}\\ 
	\subfloat[FashionMNIST]{
		\includegraphics[height=2.0cm, width=2.7cm]{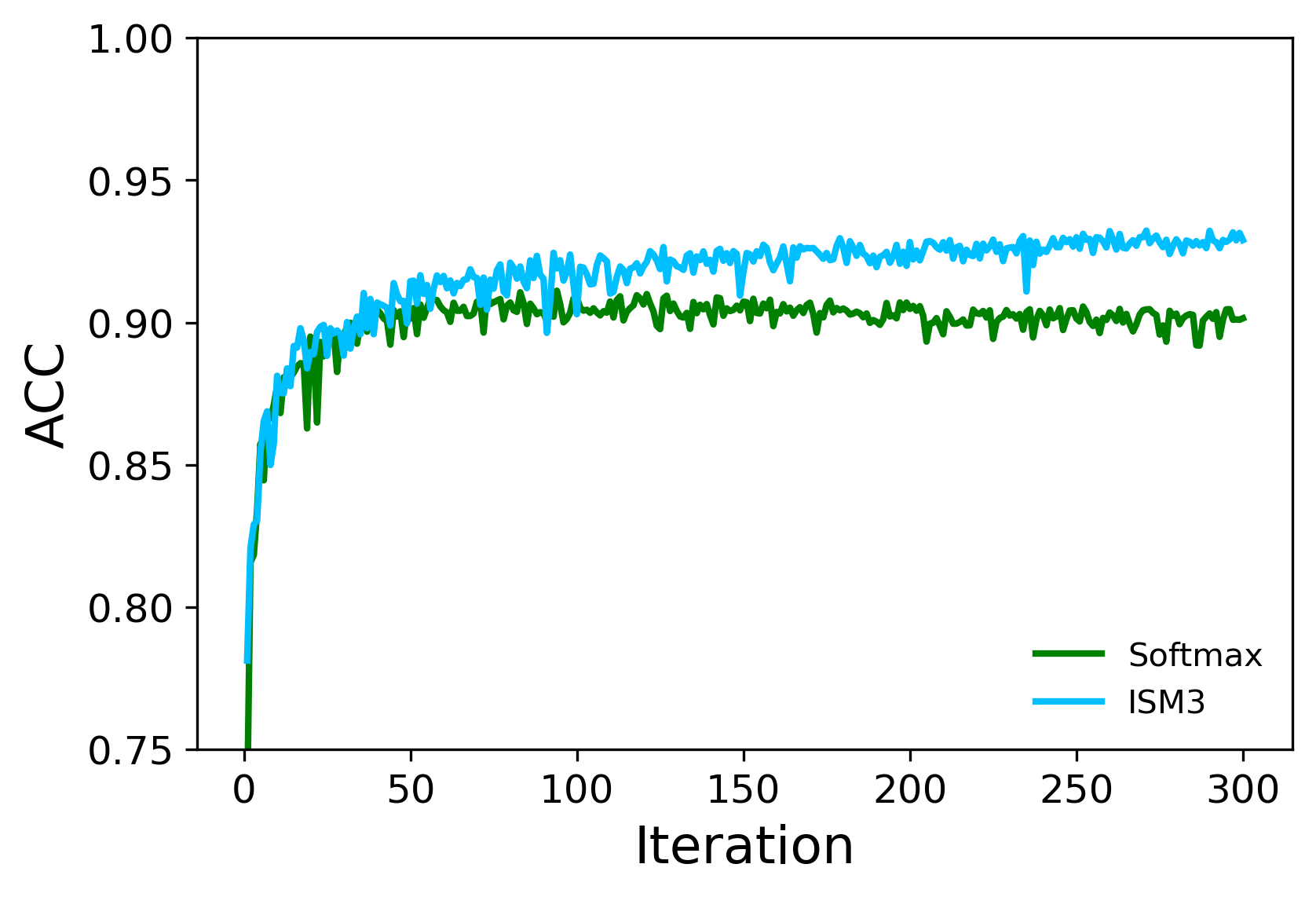} \label{initial}}
	\subfloat[SVHN]{
		\includegraphics[height=2.0cm, width=2.7cm]{SVHN_test_acc.png} \label{with high uniformity}}
	\subfloat[CIFAR10]{
		\includegraphics[height=2.0cm, width=2.7cm]{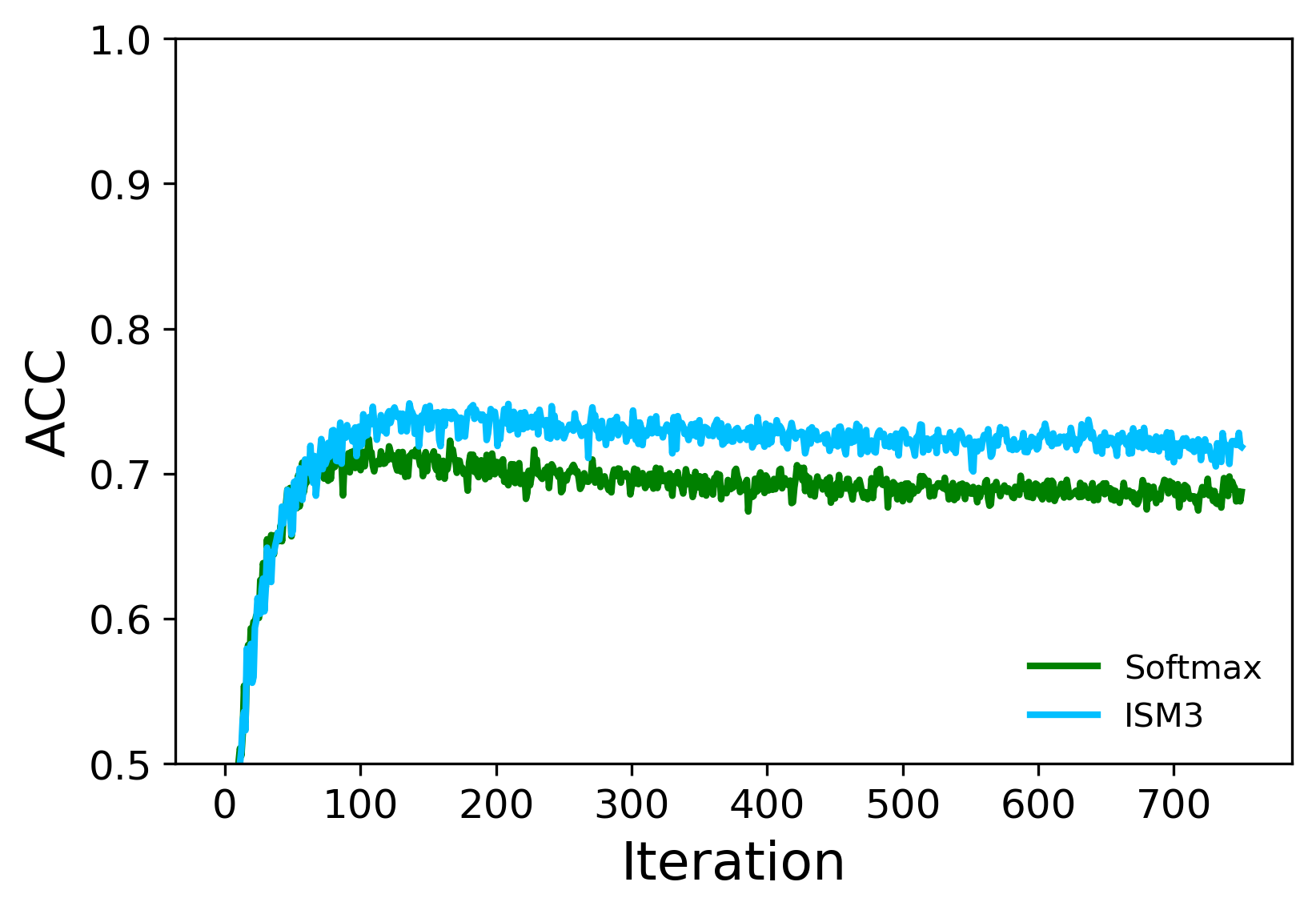} \label{with high uniformity}}\\
	\caption{Accuracy curves with iterations, training accuracy on top and test accuracy below.}
	\label{overfitting}
\end{figure}

The evolution of the training results over the number of iterations is illustrated in Figure \ref{overfitting}, with the top row displaying accuracy on the training set and the bottom row displaying accuracy on the test set.
It can be observed that even on a lightweight network, the original softmax loss still suffers from significant overfitting, as the training accuracy increases while the testing accuracy decreases. In contrast, our proposed ISM3 effectively alleviates such problem.
In circumstances where the accuracy of the training set is comparable, the test accuracy with ISM3 is noticeably improved.
Introducing a regularization term not only causes the parameter space to shrink towards the origin, but also explicitly increases the lower bound of the inter-class margin, effectively alleviating overfitting.

\subsection{A.6 Discussions}
One may wonder about the difference between our job and the previous one Xu17 \cite{Xu2017} (shorthand for convenience).
The distinctions between the two are elucidated in threefold:

\noindent  {$\diamondsuit$ \textit{Multi-class margins:}}
Both works study the unified formula derived from binary SVM; however,  Xu17 does not explain why to minimize the regular term $\sum_{k=1}^{c-1} \sum_{l=k+1}^{c} \|\mathbf{w}_{k} - \mathbf{w}_{l}\|_{2}^{2}$.
Instead, it is given directly, which has no geometric interpretation.
The objective of maximizing the sum of margins between all class pairs is
\begin{equation}\label{max sum}
	\max_{\mathbf{W} \in \mathbb{R}^{d \times c}} \sum_{k=1}^{c-1} \sum_{l=k+1}^{c}   \frac{1}{ \|\mathbf{w}_{k}-\mathbf{w}_{l}\|_{2}},
\end{equation}
which is not directly related to Xu17.
In our work, we explained that directly optimizing the ideal objective (\ref{max sum}) is impractical due to the shared parameters between hyperplanes and introduced our method.

\noindent  {$\diamondsuit$ \textit{The concept of maximizing the minimum margin:}}
Motivated by the illustration and visualization in both the main text and appendices, we converted maximizing all margins to maximizing the lower bound of margins to resolve the above delimma.
This is a conceptual innovation.
Besides, the introduced parameter $p$ in our model can be conceived as a mechanism for balancing the enlargement of overall margins and the enlargement of the lower bound of margins.

As a result, we found that the formula of Xu17 is a special case (i.e., $p = 2$) of our method, which both explains why Xu17 works and suggests a generalized form with $p \in \mathbb{R}^{+}$.
We describe the case of $p=2$ at the bottom of the left sidebar on page 5, and our experimental findings indicate that performance isn't optimal when $p=2$;  instead, empirical evidence favors $p$ around $4$.
Accordingly, the comparison with Xu17 is not in Table \ref{table:Accuracy} but in Figure \ref{p-value}.

\noindent  {$\diamondsuit$ \textit{ISM3:}} 
We introduce the margin concept into  softmax loss and explicitly enhance the lower bound of margins between classes by applying our regularizer, which is seamlessly integrable into lightweight networks.

In addition to the essential distinctions above, our work includes some non-core yet meaningful extensions, including unique solution, model properties, interpretation from SRM.
It is worth noting that the $\frac{1}{2} \sum_{j=1}^{c}\|\mathbf{w}_{j}\|_{2}^{2}$ term in the Xu17 model lacks interpretability, and we established a theorem to attain the unique solution.

It's also noteworthy that the hyper-parameters for OvR and OvO can actually be specific to each subproblem, i.e., $\{\lambda_{j}\}_{j=1}^{c}$  and $\{\lambda_{j}\}_{j=1}^{c(c-1)/2}$, respectively.
However, individually tuning them is impractical, which is also why our unified formulation surpasses OvO with less parameters.

As for the utilization of kernelization:
We recognize that M$^3$SVM is incompatible with it since the objective does not minimize each margin.
All the results are performed on linear kernel.
Despite the elegant formulations of kernel methods, they have been increasingly replaced by neural networks in practical applications.
Our ISM3 can be seamlessly applied to the output layer, regardless of the type of the preceding feature extractor, 
thus diminishing the necessity for its integration with kernelization.
Besides, since our method only regularizes the output layer, its performance cannot be compared with techniques such as dropout, layernorm, etc.
Therefore, ISM3 can be regarded as an experimental extension.

\end{document}